\documentclass{article} 
\usepackage{nips14submit_e,times, amsthm, amsmath, bbm, subfig, algorithm, algorithmic, amsfonts, amssymb, cite, graphicx, caption, enumerate, stmaryrd, ctable, cancel, soul, footnote, xspace, placeins}
\usepackage[bookmarks=false]{hyperref}

\newcommand{\gabriel}[1]{}

\newcommand{\titleName}{Fast and Robust Least Squares Estimation in Corrupted Linear Models}

\newcommand{\x}{{\mathbf x}}
\newcommand{\bE}{{\mathbb E}}

\newcommand{\uluru}{\texttt{ULURU}\xspace}
\newcommand{\srht}{\texttt{SRHT-LS}\xspace}
\newcommand{\lev}{\texttt{LEV-LS}\xspace}
\newcommand{\our}{\texttt{IWS-LS}\xspace}
\newcommand{\ourapprox}{\texttt{aIWS-LS}\xspace}
\newcommand{\irh}{\texttt{aRWS-LS}\xspace}

\newcommand{\iws}{influence weighted subsampling\xspace}

\newcommand{\OLS}{{\text{OLS}}}
\newcommand{\DS}{{\text{IWS}}}

\newcommand{\samp}{n}
\newcommand{\dims}{p}

\newcommand{\tr}{^{\top}}

\newcommand{\trace}[1]{{\rm tr}\left({#1}\right)}

\newcommand{\inv}[1]{#1^{-1}}
\newcommand{\cb}[1]{\left\{ {#1} \right\}}
\newcommand{\br}[1]{\left( {#1} \right)}
\newcommand{\sq}[1]{\left[ {#1} \right]}
\newcommand{\nrm}[1]{\Vert {#1} \Vert}

\newcommand{\order}[1]{O \br{ #1 }}

\newcommand{\Prob}{{\mathbb P}}
\newcommand{\R}{{\mathbb R}}

\newcommand{\Xt}{{\mathbf{X}}}

\newcommand{\y}{\mathbf{y}}
\newcommand{\z}{\mathbf{z}}

\newcommand{\wt}{{\mathbf{w}}}
\newcommand{\vt}{{\mathbf{v}}}

\newcommand{\et}{{\mathbf{e}}}

\newcommand{\At}{{\mathbf{A}}}

\newcommand{\Dt}{{\mathbf{D}}}

\newcommand{\Ht}{{\mathbf{H}}}

\newcommand{\It}{{\mathbf{I}}}

\newcommand{\Lt}{{\mathbf{L}}}

\newcommand{\Rt}{{\mathbf{R}}}
\newcommand{\St}{{\mathbf{S}}}
\newcommand{\Zt}{{\mathbf{Z}}}
\newcommand{\Ut}{{\mathbf{U}}}
\newcommand{\Vt}{{\mathbf{V}}}
\newcommand{\Wt}{{\mathbf{W}}}

\newcommand{\boldbeta}{\boldsymbol{\beta}}

\newcommand{\boldSigma}{\boldsymbol{\Sigma}}

\newcommand{\boldLambda}{\boldsymbol{\Lambda}}

\newcommand{\boldPi}{\boldsymbol{\Pi}}

\newcommand{\boldSigmai}{\inv{\boldsymbol{\Sigma}}}

\newcommand{\estbeta}{\widehat{\boldbeta}}
\newcommand{\estbetaSS}{\widehat{\boldbeta}_{\DS}}

\newtheorem{thm}{Theorem}
\newtheorem{cor}[thm]{Corollary}
\newtheorem{prop}[thm]{Proposition}
\newtheorem{lem}[thm]{Lemma}
\newtheorem{rem}{Remark}
\newtheorem{defn}{Definition}

\title{\titleName}

\author{
{\bf Brian McWilliams}$^{*}$ \quad {\bf Gabriel Krummenacher}\thanks{Joint first author} \quad {\bf Mario Lu\v{c}i\'c} \quad {\bf Joachim M. Buhmann}  \\ Department of Computer Science \\
ETH Z\"urich, Switzerland\\
 \texttt{\{mcbrian,gabriel.krummenacher,mario.lucic,jbuhmann\}@inf.ethz.ch}
}

\renewcommand{\algorithmicrequire}{\textbf{Input:}}
\renewcommand{\algorithmicensure}{\textbf{Output:}}

\nipsfinalcopy 

\begin{document}
\maketitle

\begin{abstract}

  \emph{Subsampling methods} have been recently proposed to speed up least
  squares estimation in large scale settings. However, these
  algorithms are typically not robust to outliers or corruptions in
  the observed covariates.

  The concept of \emph{influence} that was developed for regression
  diagnostics can be used to detect such corrupted observations as
  shown in this paper. This property of influence -- for which we also develop a randomized approximation -- motivates our
  proposed subsampling algorithm for large scale corrupted linear regression
  which limits the influence of data points since highly influential
  points contribute most to the residual error. Under a general model of corrupted observations, we show
  theoretically and empirically on a variety of simulated and real
  datasets that our algorithm improves over the current
  state-of-the-art approximation schemes for ordinary least squares.

\end{abstract}

\section{Introduction}

%


To improve scalability of the widely used ordinary least squares algorithm, a number of randomized approximation algorithms have recently been proposed. These methods, based on subsampling the dataset,
reduce the computational time from $\order{\samp\dims^2}$ to $o(\samp\dims^2)$\footnote{Informally:
  $f(n) = o(g(n))$ means $f(n)$ grows more slowly than
  $g(n)$.}
\cite{Mahoney:2011te}.  Most of these algorithms are concerned with
the classical fixed design setting or the case where the data is
assumed to be sampled i.i.d. typically from a sub-Gaussian
distribution \cite{Dhillon:2013wz}. This is known to be an unrealistic
modelling assumption since real-world data are rarely
well-behaved in the sense of the underlying distributions.


We relax this limiting assumption by considering the setting where with some
probability, the observed covariates are corrupted with additive
noise. This scenario corresponds to a generalised version of the
classical problem of ``errors-in-variables" in regression analysis
which has recently been considered in the context of sparse estimation
\cite{Loh:2012hf}.  This corrupted observation model poses a more
more realistic model of real data which may be subject to many
different sources of measurement noise or heterogeneity in the
dataset.

A key consideration for sampling is to ensure that the points used for
estimation are typical of
the full dataset. Typicality requires the sampling distribution to be robust
against outliers and corrupted points. In the
i.i.d. sub-Gaussian setting, outliers are rare and can often
easily be identified by examining the \emph{statistical leverage}
scores of the datapoints.

Crucially, in the corrupted observation setting described in \S\ref{sec:model}, the concept of an
outlying point concerns the relationship between the observed
predictors and the response. Now, leverage alone cannot detect the
presence of corruptions.
Consequently, without using additional knowledge about the corrupted
points, the OLS estimator (and its subsampled approximations) are biased. This also rules out stochastic gradient descent (SGD) -- which is often used for large scale regression -- since convex cost functions and regularizers which are typically used for noisy data are not robust with respect to measurement corruptions.

This setting motivates our use of \emph{influence} --
the effective impact of an individual datapoint exerts on the overall estimate
-- in order to detect and therefore avoid sampling corrupted
points. We propose an algorithm which is robust to corrupted observations and exhibits reduced bias compared with other
subsampling estimators.







\paragraph{Outline and Contributions.}
In \S\ref{sec:model} we introduce our corrupted observation model
before reviewing the basic concepts of statistical leverage and
influence in \S\ref{sec:diag}. In \S\ref{sec:srht} we briefly review
two subsampling approaches to approximating least squares based on
structured random projections and leverage weighted importance
sampling. Based on these ideas we present \iws (\our), a novel
randomized least squares algorithm based on subsampling points with
small influence in \S\ref{sec:alg}.

In \S\ref{sec:analysis} we analyse \our in the general setting where the observed
predictors can be corrupted with additive sub-Gaussian noise. Comparing the \our estimate with that of OLS and other
randomized least squares approaches we show a reduction in both bias
and variance. It is important to note that the simultaneous reduction
in bias and variance is relative to OLS and randomized approximations
which are only unbiased in the non-corrupted setting. Our results rely on novel finite sample characteristics of leverage and influence which we defer to \S\ref{sec:supp_lev}. Additionally, in \S\ref{sec:subgthm} we prove an estimation error bound for \our in the standard sub-Gaussian model.

Computing influence exactly is not practical in large-scale applications and so we
propose two randomized approximation algorithms based on the
randomized leverage approximation of \cite{Drineas:2011ts}. Both of
these algorithms run in $o(\samp\dims^2)$ time which improve scalability in large problems. Finally, in \S\ref{sec:results} we present extensive
experimental evaluation which compares the performance of our
algorithms against several  randomized least squares
methods on a variety of 
simulated and real datasets.

\section{Statistical model} \label{sec:model}
In this work we consider a variant of the standard linear model
\begin{equation} \label{eq:lin-mod}
\y = \Xt\boldbeta + \epsilon ,
\end{equation}
where $\epsilon \in \R^{\samp}$ is a noise term independent of
$\Xt\in\R^{\samp\times\dims}$. However, rather than directly observing
$\Xt$ we instead observe $\Zt$ where
\begin{equation}
\Zt = \Xt + U \Wt. \label{eq:obs-mod}
\end{equation}
$U = \text{diag}(u_1,\ldots,u_{\samp})$ and $u_i$ is a Bernoulli
random variable with probability $\pi$ of being 1. $\Wt\in\R^{\samp\times\dims}$ is a matrix of measurement corruptions.
The rows of $\Zt$ therefore are corrupted with probability $\pi$ and
not corrupted with probability $(1-\pi)$.

\begin{defn}[Sub-gaussian matrix]
  A zero-mean matrix $\Xt$ is called sub-Gaussian with parameter
  $(\frac{1}{n}\sigma^2_x,\frac{1}{n}\Sigma_x)$ if 
  (a) Each row $\x_i\tr\in\R^{\dims}$ is sampled independently and
    has $\bE[\x_i\x_i\tr]=\frac{1}{n}\Sigma_x$.
  (b) For any unit vector $\vt\in\R^{\dims}$, $\vt\tr\x_i$ is a
    sub-Gaussian random variable with parameter at most
    $\frac{1}{\sqrt{\dims}}\sigma_x$.
  \vspace{-0.15cm}
\end{defn}

We consider the specific instance of the linear corrupted observation model in
Eqs. \eqref{eq:lin-mod}, \eqref{eq:obs-mod} where \vspace{-0.15cm}
\begin{itemize}
\item $\Xt,\Wt \in\R^{\samp\times \dims}$ are sub-Gaussian with
  parameters $(\frac{1}{n}\sigma^2_x,\frac{1}{n}\Sigma_x)$ and
  $(\frac{1}{n}\sigma^2_w,\frac{1}{n}\Sigma_w)$ respectively,
  \vspace{-0.1cm}
\item $\epsilon\in\R^{\samp}$ is sub-Gaussian with parameters
  $(\frac{1}{n}\sigma^2_\epsilon,\frac{1}{n}\sigma^2_\epsilon\It_\samp)$,
\end{itemize}
\vspace{-0.15cm} and all are independent of each other.

The key challenge is that even when $\pi$ and the magnitude
of the corruptions, $\sigma_w$ are relatively small, the standard
linear regression estimate is biased and can perform poorly (see \S
\ref{sec:analysis}). Sampling methods which are not sensitive to
corruptions in the observations can perform even worse if they somehow
subsample a proportion $r\samp > \pi\samp$ of corrupted
points. Furthermore, the corruptions may not be large enough to be
detected via leverage based techniques alone.

The model described in this section generalises the
``errors-in-variables" model from classical least squares
modelling. Recently, similar models have been studied in the high
dimensional ($\dims\gg\samp$) setting in \cite{Chen:2012vt, Chen:up,
  Chen:2013te, Loh:2012hf} in the context of robust sparse
estimation. The ``low-dimensional" ($\samp>\dims$) setting is
investigated in \cite{Chen:2012vt}, but the ``big data" setting
($\samp\gg \dims$) has not been considered so far.\footnote{Unlike
  \cite{Chen:2013te, Loh:2012hf} and others we do not consider
  sparsity in our solution since $\samp\gg\dims$.}

In the high-dimensional problem, knowledge of the corruption covariance, $\Sigma_{w}$
\cite{Loh:2012hf}, or the data covariance $\Sigma_{x}$
\cite{Chen:2013te}, is required to obtain a consistent estimate.
This assumption may be unrealistic in many settings. 
We aim to reduce the bias in
our estimates \emph{without} requiring knowledge of the true covariance of
the data or the corruptions, and instead sub-sample only non-corrupted
points.

\section{Diagnostics for linear regression} \label{sec:diag}

In practice, the sub-Gaussian linear model assumption is often
violated either by heterogeneous noise or by a corruption model as in
\S\ref{sec:model}. In such scenarios, fitting a least squares model to
the full dataset is unwise since the outlying or corrupted points can
have a large adverse effect on the model fit. \emph{Regression
  diagnostics} have been developed in the statistics literature to
detect such points (see e.g. \cite{Belsley:1980} for a comprehensive
overview). 
Recently,
\cite{Mahoney:2011te} proposed subsampling points for least squares
based on their leverage scores. Other recent works suggest related
influence measures that identify subspace \cite{McWilliams:2012wi} and
multi-view \cite{McWilliams:2012jn} clusters in high dimensional data.



\subsection{Statistical leverage}

For the standard linear model in Eq. \eqref{eq:lin-mod}, the well
known least squares solution is
\begin{equation} \label{eq:ls}
    \estbeta = \arg\min_{\boldbeta} \nrm{\y
    - \Xt \boldbeta}^2 = \br{\Xt\tr\Xt}^{-1}\Xt\tr\y .
\end{equation}
The projection matrix $ \It - \Lt$ with $ \Lt :=
\Xt(\Xt\tr\Xt)^{-1}\Xt\tr$ specifies the subspace in which the residual
lies. The diagonal elements of the ``hat matrix" $\Lt$, $l_i :=
L_{ii}$, $i=1,\ldots,\samp$ are the \emph{statistical leverage} scores of the
$i^{th}$ sample. Leverage scores quantify to what extent a
particular sample is an outlier with respect to the distribution of $\Xt$.

An equivalent definition from \cite{Mahoney:2011te} which will be useful later concerns any matrix $\Ut\in\R^{\samp\times\dims}$ which spans the column space of $\Xt$ (for example, the matrix whose columns are the left singular vectors of $\Xt$). The statistical leverage scores of the rows of $\Xt$ are the squared row norms of $\Ut$, i.e. $l_i=\nrm{\Ut_i}^2$.


Although the use of leverage can be motivated from the least squares solution in Eq. \eqref{eq:ls}, the leverage scores do not take into account the relationship between the predictor variables and the response variable $\y$. Therefore, low-leverage points may have a weak predictive relationship with the response and vice-versa. In other words, it is possible for such points to be outliers with respect to the conditional distribution $p(\y | \Xt)$ but not the marginal distribution on $\Xt$.

\subsection{Influence}

A concept that captures the predictive relationship between covariates and response is \emph{influence}. Influential points are  those that might not be outliers in the geometric
sense, but instead adversely affect the estimated coefficients.
One way to assess the influence of a point is to compute the
change in the learned model when the point is removed
from the estimation step. \cite{Belsley:1980}.

We can compute a leave-one-out least squares estimator by straightforward application of the Sherman-Morrison-Woodbury formula (see Prop. \ref{prop:betai} in \S\ref{sec:supp_lev}):
\begin{align*}
\estbeta_{-i} & = \br{\Xt\tr\Xt - \x_i\tr\x_i}^{-1}\br{\Xt\tr\y - \x_i\tr y_i} = \estbeta - \frac{\boldSigmai\x_i\tr e_i}{1-l_i}
\end{align*}
where $e_i = y_i - \x_i\estbeta_{\OLS}$.
Defining the influence\footnote{The expression we use is also called \emph{Cook's distance} \cite{Belsley:1980}.}, $d_i$ as the change in expected mean squared error we have
\begin{align*}
d_i & = \br{ \estbeta - \estbeta_{-i}}\tr \Xt\tr\Xt \br{ \estbeta - \estbeta_{-i}} = \frac{e_i^2 l_i}{\br{1-l_i}^2} .
\end{align*}
Points with large values of $d_i$ are those which, if added to the model, have the largest adverse effect on the resulting estimate. Since influence only depends on the OLS residual error and the leverage scores, it can be seen that the influence of every point can be computed at the cost of a least squares fit. In the next section we will see how to approximate both quantities using random projections. 



%
%
%
%

\section{Fast randomized least squares algorithms} \label{sec:srht}

We briefly review two randomized approaches to least squares approximation:
the importance weighted subsampling approach of
\cite{Drineas:2006iw} and the dimensionality reduction approach
\cite{Mahoney:2011te}. The former proposes an importance sampling
probability distribution according to which, a small number of rows of
$\Xt$ and $\y$ are drawn and used to compute the regression
coefficients. If the sampling probabilities are
proportional to the statistical leverages, the resulting estimator is
close to the optimal estimator \cite{Drineas:2006iw}. We refer to this as \lev.

The dimensionality reduction approach can be viewed as a random projection
step followed by a uniform subsampling. The class of Johnson-Lindenstrauss projections -- e.g. the SRHT -- has been shown to approximately uniformize leverage scores in the projected space. 
Uniformly subsampling the rows of the projected
matrix proves to be equivalent to leverage weighted sampling on the
original dataset \cite{Mahoney:2011te}. We refer to this as \srht.
It is analysed in the statistical setting by
\cite{Dhillon:2013wz} who also propose \uluru, a two step
fitting procedure which aims to correct for the subsampling bias and
consequently converges to the OLS estimate at a rate independent
of the number of subsamples \cite{Dhillon:2013wz}.

%

\paragraph{Subsampled Randomized Hadamard Transform
  (SRHT)} \label{sec:srht2} The SHRT consists of a preconditioning
step after which $\samp_{subs}$ rows of the new matrix are subsampled
uniformly at random in the following way
$\sqrt{\frac{\samp}{\samp_{subs}}}\St \Ht \Dt \cdot \Xt = \boldPi \Xt$
 with the definitions \cite{Boutsidis:2012tv}:
\vspace{-0.5cm}\begin{itemize}
\item $\St$ is a subsampling matrix. \vspace{-0.1cm}
\item $\Dt$ is a diagonal matrix whose entries are drawn independently
  from $\{-1, 1\}$. \vspace{-0.1cm}
\item $\Ht \in \R^{\samp \times \samp}$ is a normalized Walsh-Hadamard
  matrix\footnote{For the Hadamard transform, $\samp$ must be a power
    of two but other transforms exist (e.g. DCT, DFT)
    for which similar theoretical guarantees hold and there is no
    restriction on $\samp$.} which is defined recursively as
$$
\Ht_n = \sq{\begin{array}{cc} \Ht_{n/2} & \Ht_{n/2} \\ \Ht_{n/2} & -\Ht_{n/2} \end{array} }
,~~
\Ht_2 = \sq{\begin{array}{cc} +1 & +1 \\ +1 & -1 \end{array} }.
$$
We set $\Ht = \frac{1}{\sqrt{\samp}}\Ht_\samp$ so it has orthonormal
columns.
\end{itemize}
As a result, the rows of the transformed matrix
$\boldPi\Xt$ have approximately uniform leverage scores.
(see \cite{Tropp:2010uo} for detailed analysis
of the SRHT).
Due to the recursive nature of $\Ht$, the cost of applying the SRHT is
$\order{\dims\samp\log \samp_{subs}}$ operations, where $\samp_{subs}$
is the number of rows sampled from $\Xt$ \cite{Ailon:2008}.

%

The \srht algorithm solves
$\estbeta_{SRHT} = \arg\min_{\boldbeta} \nrm{\boldPi \y - \boldPi \Xt\boldbeta}^2$
which for an appropriate subsampling ratio,
$r=\Omega(\frac{\dims^2}{\rho^2})$ results in a residual error,
$\tilde{\et}$ which satisfies
\begin{equation} \label{eq:randLSerror}
\nrm{\tilde{\et}} \leq (1+\rho) \nrm{{\et}}
\end{equation}
where ${\et} = \y - \Xt\estbeta_{\OLS}$ is the vector of OLS residual
errors \cite{Mahoney:2011te}.



\paragraph{Randomized leverage computation} \label{sec:randlev}



Recently, a method based on random projections has been proposed to approximate the
leverage scores based on first reducing the dimensionality of the data using the SRHT
followed by computing the leverage scores using this low-dimensional
approximation \cite{Drineas:2011ts, Ma:2013tx,
  Drineas:2011ih,Drineas:2006iw}.



The leverage approximation algorithm of
\cite{Drineas:2011ts} uses a SRHT, $\boldPi_1\in\R^{r_1\times \samp}$
to first compute the approximate SVD of $\Xt$,
%
%
%
%
$\boldPi_1\Xt = \Ut_{\Pi X}\boldSigma_{\Pi X}\Vt_{\Pi X}\tr.$ Followed
by a second SHRT $\boldPi_2\in\R^{\dims\times r_2}$
to compute an approximate orthogonal basis for $\Xt$
\begin{align} \label{eq:proj2}
\Rt^{-1} = \Vt_{\Pi X}\boldSigma^{-1}_{\Pi X} \in \R^{\dims \times \dims} ,
~~
\tilde{\Ut} = \Xt \Rt^{-1}\boldPi_2 \in \R^{\samp \times r_2} .
\end{align}
The approximate leverage scores are now the squared row norms of $\tilde{\Ut}$,
$
\tilde{l}_i = \nrm{\tilde{\Ut}_i}^2 .
$
%
%
From \cite{Mahoney:2011te} we derive the following result relating to
randomized approximation of the leverage
\begin{equation} \label{eq:levapp}
\tilde{l}_i \leq (1+\rho_l)l_i \,,
\end{equation}
where the approximation error, $\rho_l$ depends on the choice of
projection dimensions $r_1$ and $r_2$.

The leverage weighted least squares (\lev) algorithm samples rows of
$\Xt$ and $\y$ with probability proportional to $l_i$ (or
$\tilde{l}_i$ in the approximate case) and performs least squares on
this subsample. The residual error resulting from the leverage
weighted least squares is bounded by Eq. \eqref{eq:randLSerror}
implying that \lev and \srht are equivalent \cite{Mahoney:2011te}.
It is important to note that under the corrupted observation model these approximations will be biased.

\section{Influence weighted subsampling} \label{sec:alg}



In the corrupted observation model, OLS and therefore the random approximations to OLS described in \S\ref{sec:randlev} obtain poor predictions.
To remedy this, we propose \iws (\our) which is described in Algorithm
\ref{alg:dsamp}.  \our subsamples points according to the
distribution, $p_i = c /d_i$ where $c$ is a normalizing constant so
that $\sum_{i=1}^\samp p_i = 1$. OLS is then estimated on the
subsampled points. The sampling procedure ensures that points with
high influence are selected infrequently and so the resulting estimate
is less biased than the full OLS
solution.

Obviously, \our is impractical in the scenarios we consider since it
requires the OLS residuals and full leverage scores. However, we use
this as a baseline and to simplify the
analysis.
In the next section, we propose an approximate \iws algorithm which
combines the approximate leverage computation of \cite{Drineas:2011ts}
and the randomized least squares approach of \cite{Mahoney:2011te}.


%
\vspace{-0.5cm}
\begin{minipage}[t]{0.5\textwidth}
\begin{algorithm}[H]
\caption{Influence weighted subsampling (\our).\label{alg:dsamp}}
	\algorithmicrequire\; Data: $\Zt$, $\y$
  \begin{algorithmic}[1]
    \STATE {\bf \emph{Solve}} $\estbeta_{\OLS}=\arg\min_{\boldbeta} \nrm{\y - \Zt\boldbeta}^2$
    \FOR{$i=1\ldots\samp$}
           \STATE $e_i = y_i - \z_i\estbeta_{\OLS}$ 
           \STATE $l_i = \z_i\tr(\Zt\tr\Zt)^{-1}\z_i$
           \STATE $d_i = e_i^2 l_i/(1-l_i)^2$
    \ENDFOR
    \STATE {\bf \footnotesize \emph{Sample rows ($\tilde{\Zt}$, $\tilde{\y}$) of ($\Zt$, $\y$) proportional to $\frac{1}{d_i}$}}
    \STATE {\bf \emph{Solve}} $\estbetaSS = \arg\min_{\boldbeta} \nrm{\tilde{\y} - \tilde{\Zt}\boldbeta}^2$
  \end{algorithmic}
  \algorithmicensure\; $\estbetaSS$
\end{algorithm}
\end{minipage}
\begin{minipage}[t]{0.5\textwidth}
\begin{algorithm}[H]
\caption{Residual weighted subsampling (\irh) \label{alg:res}}
	\algorithmicrequire\; Data: $\Zt$, $\y$
  \begin{algorithmic}[1]
	\STATE {\bf \footnotesize \emph{Solve $\estbeta_{SRHT} = \arg\min_{\boldbeta} \nrm{\boldPi\cdot\br{\y -  \Zt\boldbeta}}^2$}}
    \STATE {\bf \footnotesize \emph{Estimate residuals: $\tilde{\et} = \y -\Zt \estbeta_{SRHT}$}}
    \STATE {\bf \emph{Sample rows ($\tilde{\Zt}$, $\tilde{\y}$) of ($\Zt$, $\y$) proportional to $\frac{1}{\tilde{e}_i^2}$}}
  	\STATE {\bf \emph{Solve $\estbeta_{RWS} = \arg\min_{\boldbeta} \nrm{\tilde{\y} -\tilde{\Zt}\boldbeta }^2$}}
  \end{algorithmic}
  \algorithmicensure\; $\estbeta_{RWS}$
\end{algorithm} 
\end{minipage}


\paragraph{Randomized approximation algorithms.}
Using the ideas from \S\ref{sec:srht2} and \S\ref{sec:randlev} we
obtain the following randomized approximation to the influence scores
\begin{equation} \label{eq:appInf} \tilde{d}_i = \frac{\tilde{e}_i^2
    \tilde{l}_i}{(1-\tilde{l}_i)^2} ,
\end{equation}
where $\tilde{e}_i$ is the $i^{th}$ residual error computed using the
\srht estimator.  Since the approximation errors of $\tilde{e}_i$ and
$\tilde{l}_i$ are bounded (inequalities \eqref{eq:randLSerror} and
\eqref{eq:levapp}), this suggests that our randomized approximation to
influence is close to the true influence.

\paragraph{Basic approximation.}
The first approximation algorithm is identical to Algorithm
\ref{alg:dsamp} except that leverage and residuals are replaced by
their randomized approximations as in Eq. \eqref{eq:appInf}. We refer
to this algorithm as Approximate \iws (\ourapprox). Full details are
given in Algorithm \ref{alg:dsampapprox} in \S\ref{sec:suppalg}.

\paragraph{Residual Weighted Sampling.}
Leverage scores are typically uniform \cite{Ma:2013tx,Dhillon:2013wz}
for sub-Gaussian data. Even in the corrupted setting, the difference
in leverage scores between corrupted and non-corrupted points is small
(see \S\ref{sec:analysis}).  Therefore, the main contribution to the influence for each point will originate
from the residual error, $e_i^2$. Consequently, we propose sampling
with probability inversely proportional to the approximate residual,
$\frac{1}{\tilde{e}_i^2}$.  The resulting algorithm Residual Weighted
Subsampling (\irh) is detailed in Algorithm \ref{alg:res}. Although
\irh is not guaranteed to be a good approximation to \our, empirical
results suggests that it works well in practise and is faster to
compute than \ourapprox.

\paragraph{Computational complexity.} Clearly, the computational
complexity of \our is $\order{\samp \dims^2}$. The computation
complexity of \ourapprox is $\order{\samp\dims\log \samp_{subs} +
  \samp\dims r_2 + \samp_{subs}\dims^2}$, where the first term is the
cost of \srht, the second term is the cost of approximate leverage
computation and the last term solves OLS on the subsampled
dataset. Here, $r_2$ is the dimension of the random projection
detailed in Eq. \eqref{eq:proj2}. The cost of \irh is
$\order{\samp\dims\log \samp_{subs} + \samp\dims +
  \samp_{subs}\dims^2}$ where the first term is the cost of \srht, the
second term is the cost of computing the residuals $\et$, and the last
term solves OLS on the subsampled dataset. This computation can be
reduced to $\order{\samp\dims\log \samp_{subs} +
  \samp_{subs}\dims^2}$. Therefore the cost of both \ourapprox and \irh is $o(\samp\dims^2)$.

%

\section{Estimation error} \label{sec:analysis}

In this section we will prove an upper bound on the estimation error of \our in the corrupted model.
First, we show that the OLS error consists of two additional variance terms that depend on the size and proportion of the corruptions and an additional bias term. We then show that \our can significantly reduce the relative variance and bias in this setting, so that it no longer depends on the magnitude of the corruptions but only on their proportion. We compare these results to recent results from \cite{Loh:2012hf,Chen:2012vt} suggesting that consistent estimation requires knowledge about $\Sigma_w$. More recently, \cite{Chen:2013te} show that incomplete knowledge about this quantity results in a biased estimator where the bias is proportional to the uncertainty about $\Sigma_w$. We see that the form of our bound matches these results. 

Inequalities are said to hold \emph{with high probability (w.h.p.)} if the probability of failure is not more than $C_1 \exp(-C_2\log\dims)$ where $C_1, C_2$ are positive constants that do not depend on the scaling quantities $\samp,\dims,\sigma_w$. The symbol $\lesssim$ means that we ignore constants that do not depend on these scaling quantities.
Proofs are provided in the supplement. Unless otherwise stated, $\nrm{\cdot}$ denotes the $\ell_2$ norm for vectors and the spectral norm for matrices.

\paragraph{Corrupted observation model.}
As a baseline, we first investigate the behaviour of the OLS estimator in the corrupted model.
\begin{thm}[A bound on $\nrm{\estbeta_{OLS} - \boldbeta}$] \label{thm:corruptedls}
If $\samp\gtrsim \frac{\sigma_x^2\sigma_w^2}{\lambda_{\min}(\Sigma_x)} \dims\log \dims$ then w.h.p.
\begin{align}
\nrm{\estbeta_{\OLS} - \boldbeta}  \lesssim & \br{ 
\br{\sigma_\epsilon \sigma_x   
+ \pi  \sigma_\epsilon\sigma_w  
+ \pi \br{\sigma_w^2 + \sigma_w\sigma_x }\nrm{\boldbeta}}\sqrt{\frac{\dims \log \dims}{\samp}}
+ 
\pi\sigma_w^2\sqrt{\dims}\nrm{\boldbeta}  
}
 \cdot \frac{1}{\lambda}
\label{eq:thm4}
\end{align}
where $0 < \lambda \leq \lambda_{\min}(\Sigma_x) + \pi\lambda_{\min}(\Sigma_w)$.
\end{thm}
\begin{rem}[No corruptions case]
Notice for a fixed $\sigma_w$, taking $\lim_{\pi\rightarrow 0}$ or for a fixed $\pi$ taking $\lim_{\sigma_w\rightarrow 0}$  (i.e. there are no corruptions) 
the above error reduces to the least squares result (see for example \cite{Chen:2012vt}).
\end{rem}

\begin{rem}[Variance and Bias]
The first three terms in \eqref{eq:thm4} scale with $\sqrt{1/\samp}$ so as $\samp\rightarrow\infty$, these terms tend towards 0. The last term does not depend on $\sqrt{1/\samp}$ and so for some non-zero $\pi$ the least squares estimate will incur some bias depending on the fraction and magnitude of corruptions.
\end{rem}

We are now ready to state our theorem characterising the mean squared error of the \iws estimator.
\begin{thm}[Influence sampling in the corrupted model] \label{thm:corrupt}
For $\samp\gtrsim \frac{\sigma_x^2\sigma_w^2}{\lambda_{\min}(\Sigma_{\Theta x})}\dims\log \dims$ we have
\begin{align*}
\nrm{\estbetaSS - \boldbeta}  \lesssim &  \br{ \br{\sigma_\epsilon \sigma_x    
+  \frac{\pi \sigma_\epsilon}{(\sigma_w+1)} 
+  \pi \nrm{\boldbeta} }\sqrt{\frac{\dims \log \dims}{\samp_{subs}}}
+ \pi \sqrt{\dims} \nrm{\boldbeta}
} 
. \frac{1}{\lambda}
 \end{align*}
where $0 < \lambda \leq \lambda_{\min}(\Sigma_{\Theta x})$ and $\Sigma_{\Theta x}$ is the covariance of the influence weighted subsampled data.
\end{thm}
\begin{rem}
Theorem \ref{thm:corrupt} states that the \iws estimator removes the proportional dependance of the error on $\sigma_w$ so the  additional variance terms scale as $O(\pi/\sigma_w \cdot \sqrt{\dims/\samp_{subs}})$ and $O(\pi\sqrt{\dims/\samp_{subs}})$. The relative contribution of the bias term is $\pi\sqrt{\dims}\nrm{\boldbeta}$ compared with $\pi\sigma_w^2\sqrt{\dims}\nrm{\boldbeta}$ for the OLS or non-influence-based subsampling methods.
\end{rem}

\paragraph{Comparison with fully corrupted setting.}
We note that the bound in Theorem \ref{thm:corruptedls} is similar to the bound in \cite{Chen:2013te} for an estimator where all data points are corrupted (i.e. $\pi=1$) and where incomplete knowledge of the covariance matrix of the corruptions, $\Sigma_w$ is used. The additional bias in the estimator is proportional to the uncertainty in the estimate of $\Sigma_w$ -- in Theorem \ref{thm:corruptedls} this corresponds to $\sigma_w^2$. Unbiased estimation is possible if $\Sigma_w$ is known. See the Supplementary Information for further discussion, where the relevant results from \cite{Chen:2013te} are provided in Section~\ref{sec:disc} as Lemma \ref{lem:chen2}.

\begin{figure}[!tp]
\vspace{-20pt}
\begin{centering}
\subfloat[Influence (1.1)]{
\begin{minipage}{0.25\columnwidth}
\includegraphics[width=0.99\textwidth,keepaspectratio=false]{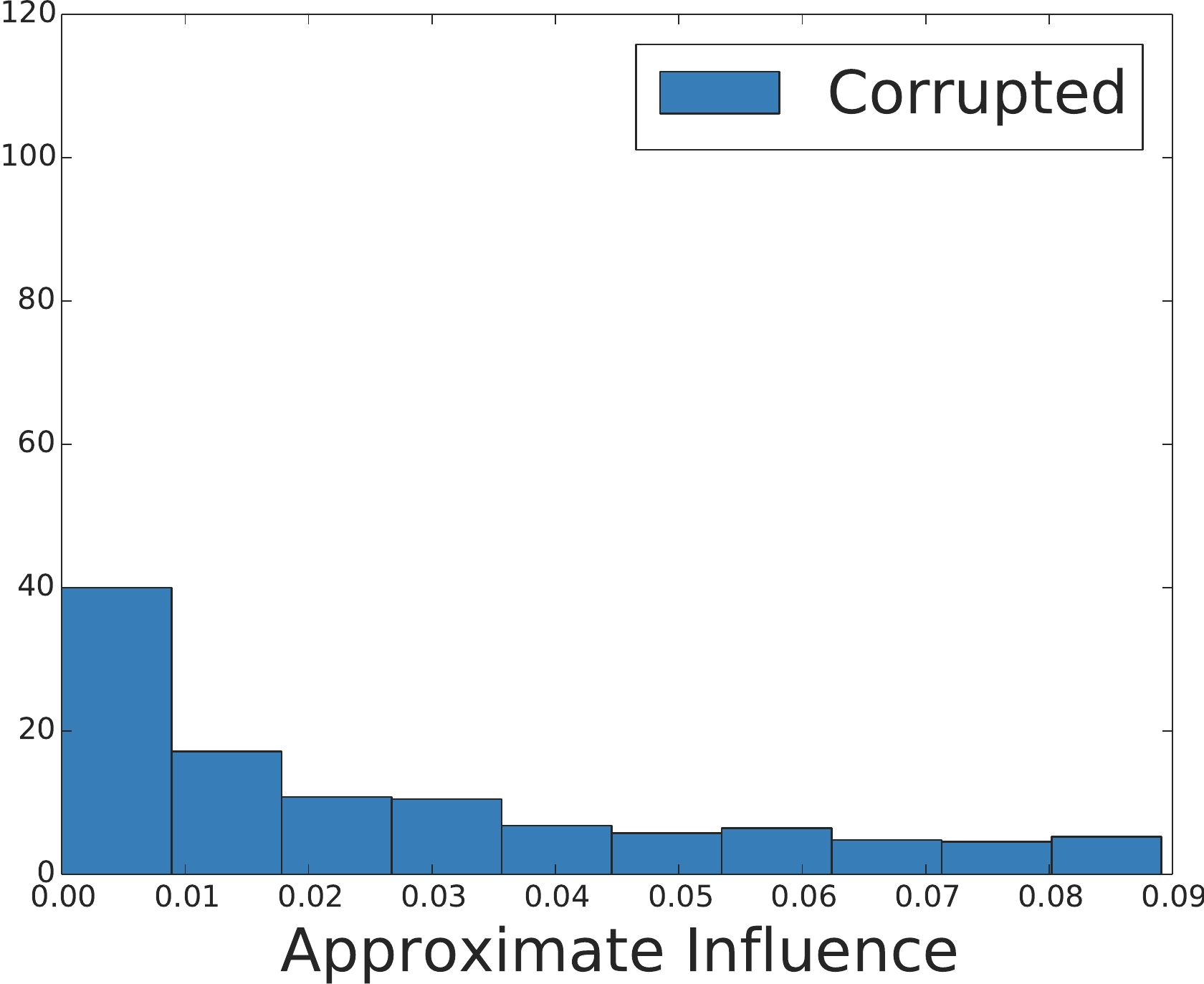}
\end{minipage}
\begin{minipage}{0.25\columnwidth}
\includegraphics[width=0.99\columnwidth,keepaspectratio=false]{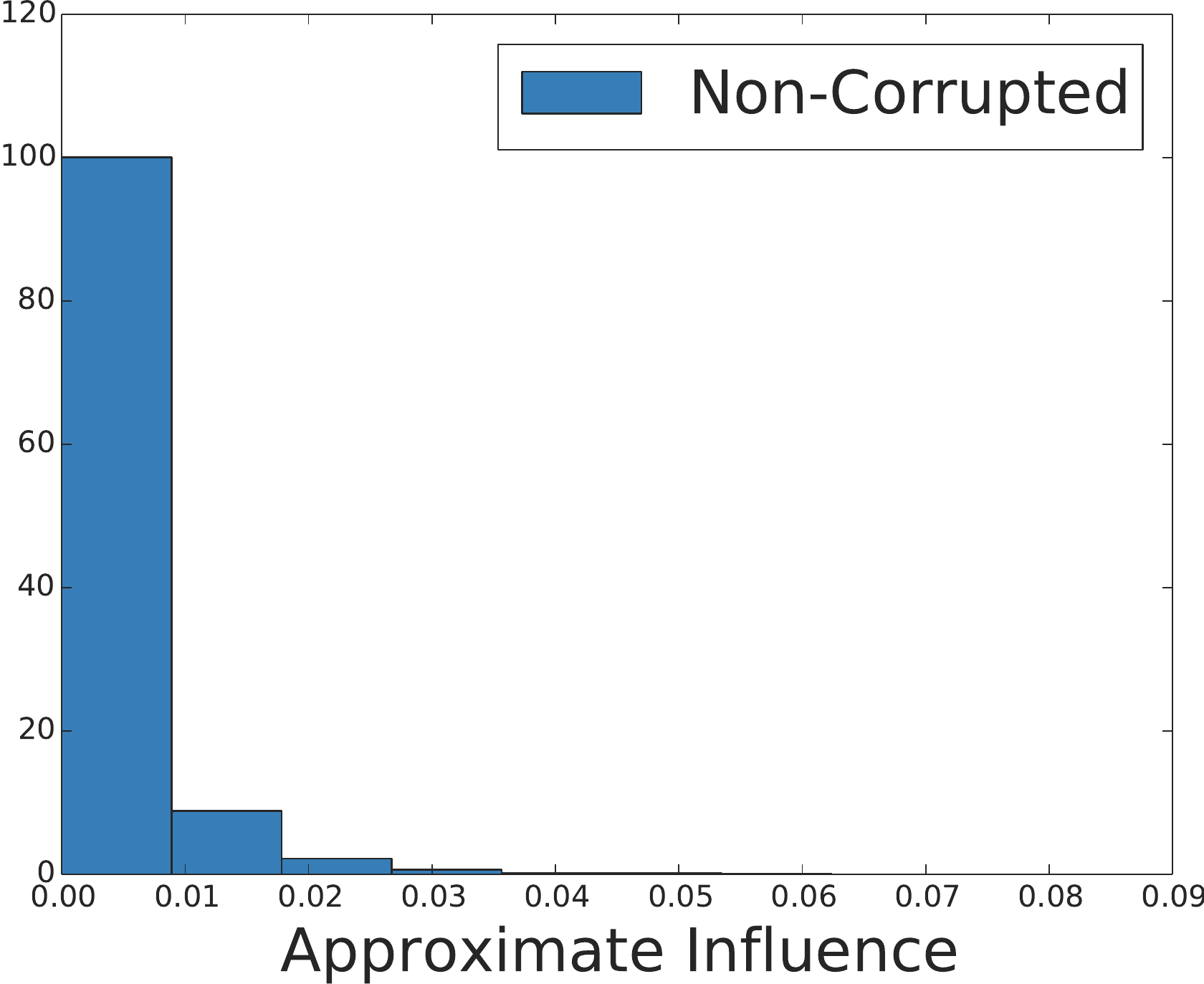}
\end{minipage}
}
\subfloat[Leverage (0.1)]{
\begin{minipage}{0.25\columnwidth}
\includegraphics[width=0.99\textwidth,keepaspectratio=false]{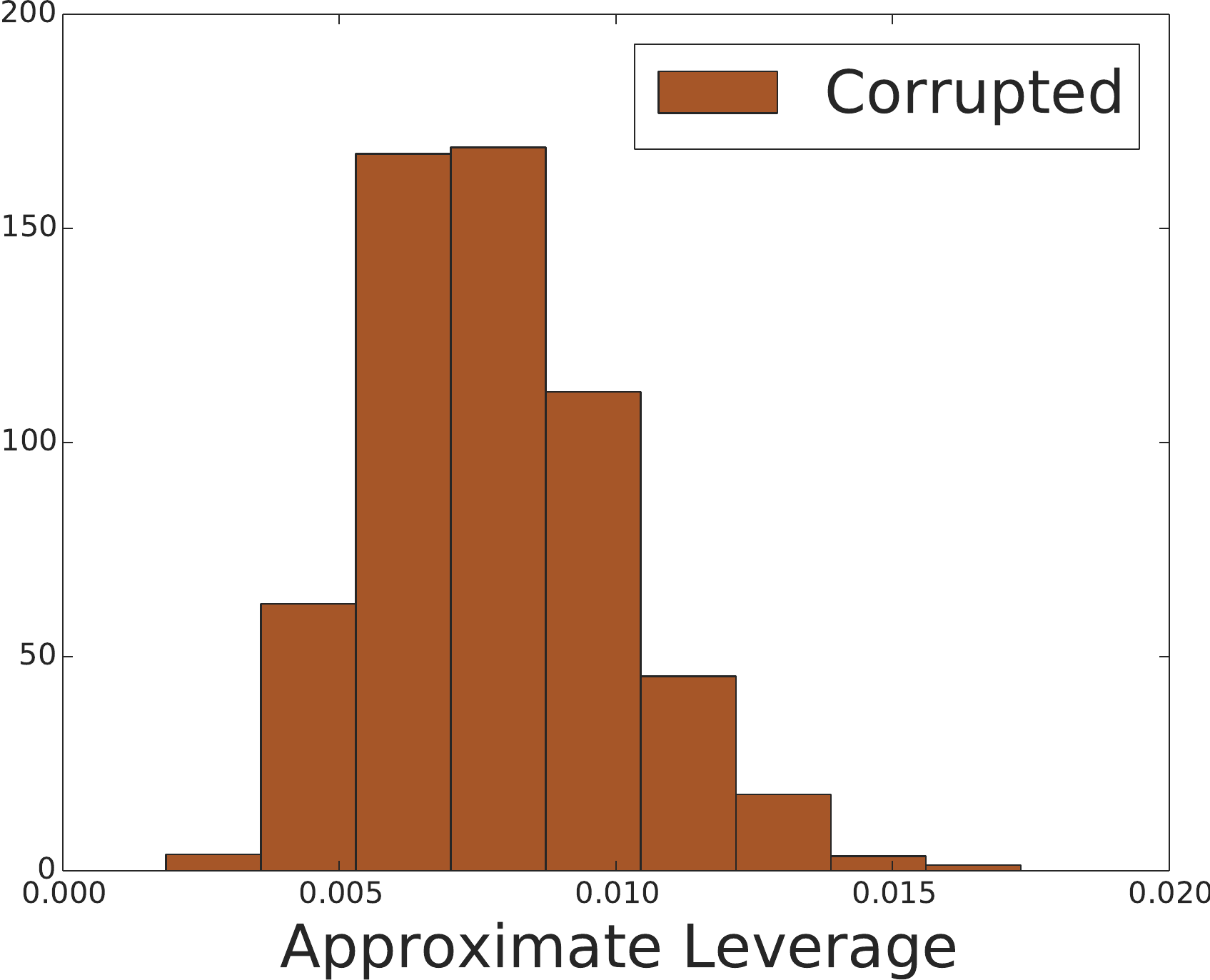}
\end{minipage}
\begin{minipage}{0.25\columnwidth}
\includegraphics[width=0.99\columnwidth,keepaspectratio=false]{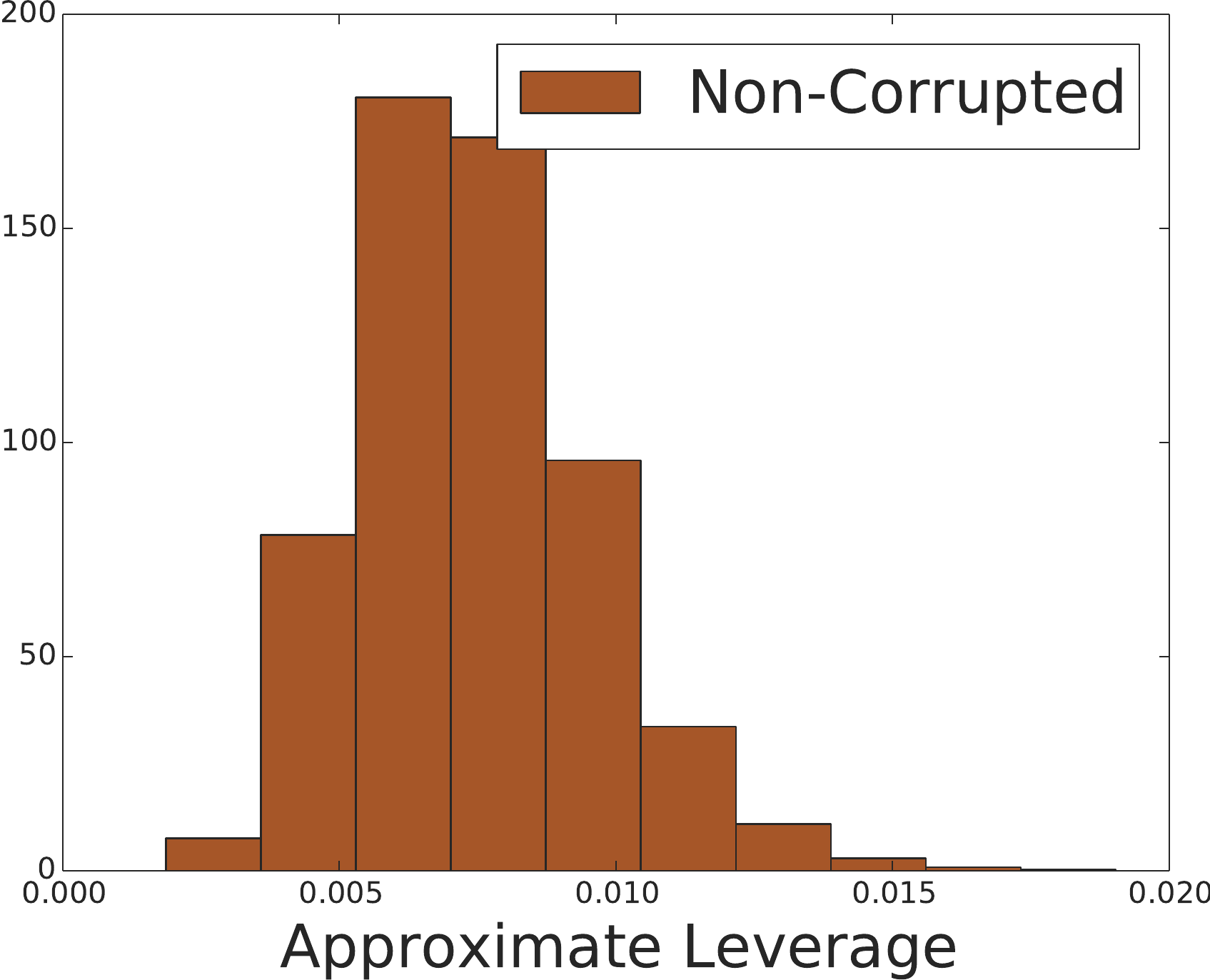}
\end{minipage}
}
\vspace{-10pt}
\caption{Comparison of the distribution of the influence and leverage for corrupted and non-corrupted points. To quantify the difference in these distributions, the $\ell_1$ distance between the histograms is shown in brackets. \label{fig:r1}}
\end{centering}
\vspace{-10pt}
\end{figure}
\section{Experimental results} \label{sec:results}
\vspace{-5pt}
We compare \our against the methods \srht
\cite{Mahoney:2011te}, \uluru \cite{Dhillon:2013wz}. These competing
methods represent current state-of-the-art in fast randomized least
squares. Since \srht is equivalent to \lev
\cite{Drineas:2006iw} the comparison will highlight the difference
between importance sampling according to the two difference types of
regression diagnostic in the corrupted model.
Similar to \our, \uluru is also a two-step procedure where the
first is equivalent to \srht. The second reduces bias by subtracting
the result of regressing onto the residual. The experiments with the
corrupted data model will demonstrate the difference in robustness of
\our and \uluru to corruptions in the observations. Note that we do not compare with SGD. Although SGD has excellent properties for large-scale linear regression, we are not aware of a convex loss function which is robust to the corruption model we propose.

We  assess the empirical
performance of our method compared with standard and state-of-the-art
randomized approaches to linear regression in several difference scenarios. We evaluate these methods
on the basis of the estimation error:
the $\ell_2$ norm of the difference between the true weights and the
learned weights, $\nrm{\estbeta - \boldbeta}$. We present additional results for root mean squared prediction error (RMSE) on the test set in \S\ref{sec:addres}.


%
For all the experiments on simulated data sets we use $\samp_{train}=100,000$,  $\samp_{test}=1000$, $\dims=500$. For datasets of this size, computing exact leverage is impractical and so we report on results for \our in \S\ref{sec:addres}. For \ourapprox and \irh we used the same number of sub-samples to approximate the leverage scores and residuals as for solving the regression. For \ourapprox we set $r_2=\dims/2$ (see Eq. \eqref{eq:proj2}). The results are averaged over 100 runs.

\vspace{-10pt}
\paragraph{Corrupted data.}
We investigate the corrupted data noise model described in Eqs. \eqref{eq:lin-mod}-\eqref{eq:obs-mod}. We show three scenarios where $\pi=\{ 0.05, 0.1, 0.3 \}$. $\Xt$ and $\Wt$ were sampled from independent, zero-mean Gaussians with standard deviation $\sigma_x = 1$ and $\sigma_w = 0.4$ respectively. The true regression coefficients, $\boldbeta$ were sampled from a standard Gaussian. We added i.i.d. zero-mean Gaussian noise with standard deviation $\sigma_e = 0.1$.

Figure~\ref{fig:r1} shows the difference in distribution of influence and leverage between non-corrupted points (top) and corrupted points (bottom) for a dataset with 30\% corrupted points. The distribution of leverage is very similar between the corrupted and non-corrupted points, as quantified by the $\ell_1$ difference. This suggests that leverage alone cannot be used to identify corrupted points. On the other hand, although there are some corrupted points with small influence, they typically have a much larger influence than non-corrupted points. We give a theoretical explanation of this phenomenon in \S\ref{sec:supp_lev} (remarks \ref{rem:comp_lev} and \ref{rem:comp_inf}).


Figure \ref{fig:Corrupted}\subref{fig:c1} and \subref{fig:c2} shows the estimation error and the mean squared prediction error for different subsample sizes. In this setting, computing \our is impractical (due to the exact leverage computation) so we omit the results but we notice that \ourapprox and \irh quickly improve over the full least squares solution and the other randomized approximations in all simulation settings. In all cases, influence based methods also achieve lower-variance estimates.  


For $30\%$ corruptions for a small number of samples \uluru outperforms the other subsampling methods. However, as the  number of samples increases, influence based methods start to outperform OLS.
 Here, \uluru converges quickly to the OLS solution but is not able to overcome the bias introduced by the corrupted datapoints. Results for $10\%$ corruptions are shown in Figs. \ref{fig:Corrupted2_est} and \ref{fig:Corrupted2_rmse} and we provide results on smaller corrupted datasets (to show the performance of \our) as well as non-corrupted data simulated according to \cite{Ma:2013tx} in \S\ref{sec:addres}.  


\begin{figure*}[!tp]
\vspace{-20pt}
\begin{centering}
\subfloat[{5\% Corruptions}]{
\begin{minipage}{0.33\textwidth}
    \includegraphics[width=0.98\textwidth, keepaspectratio=true]{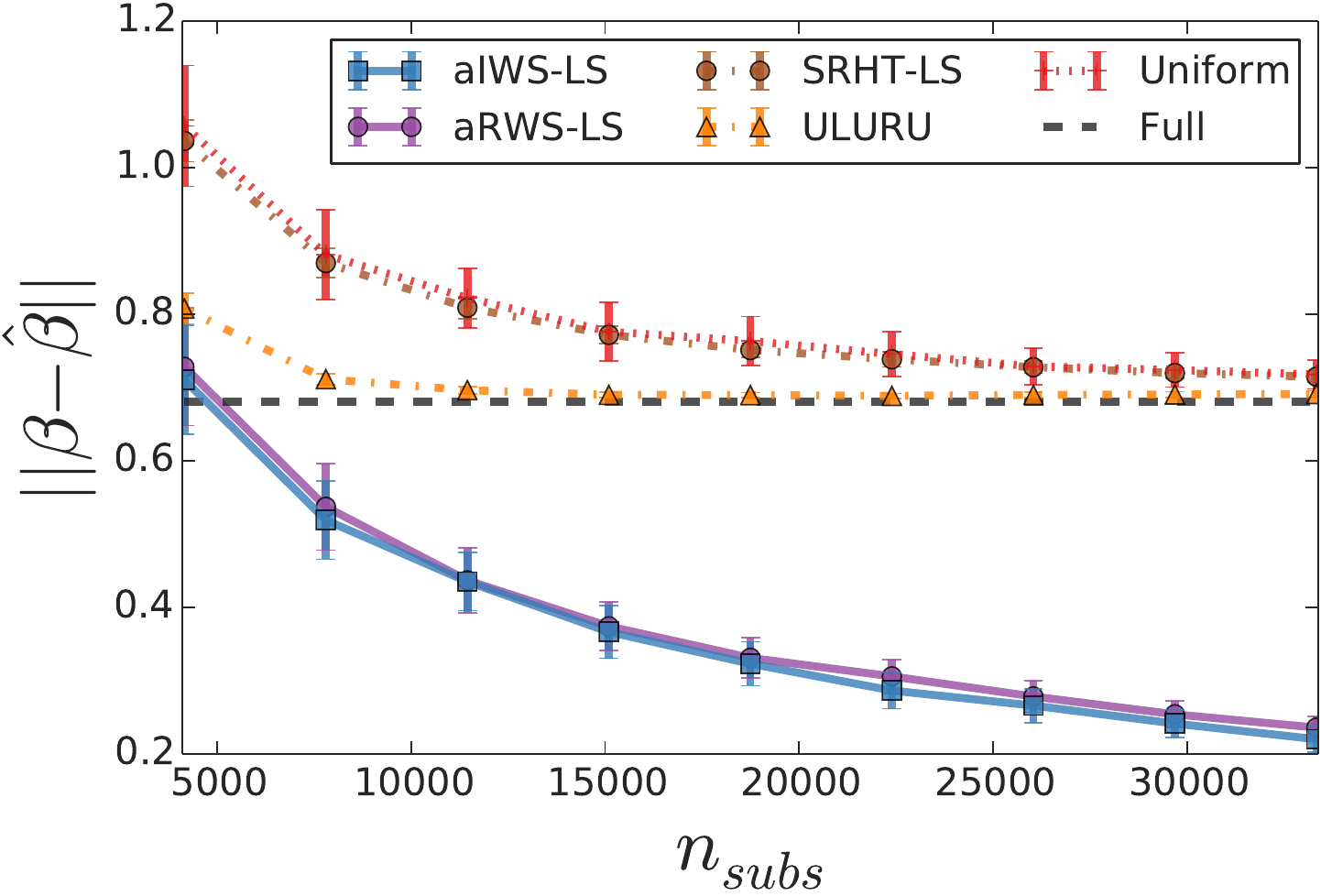}

    \label{fig:c1}
\end{minipage}
}
%
\subfloat[{30\% Corruptions}]{
\begin{minipage}{0.33\textwidth}
    \includegraphics[width=0.98\textwidth, keepaspectratio=true]{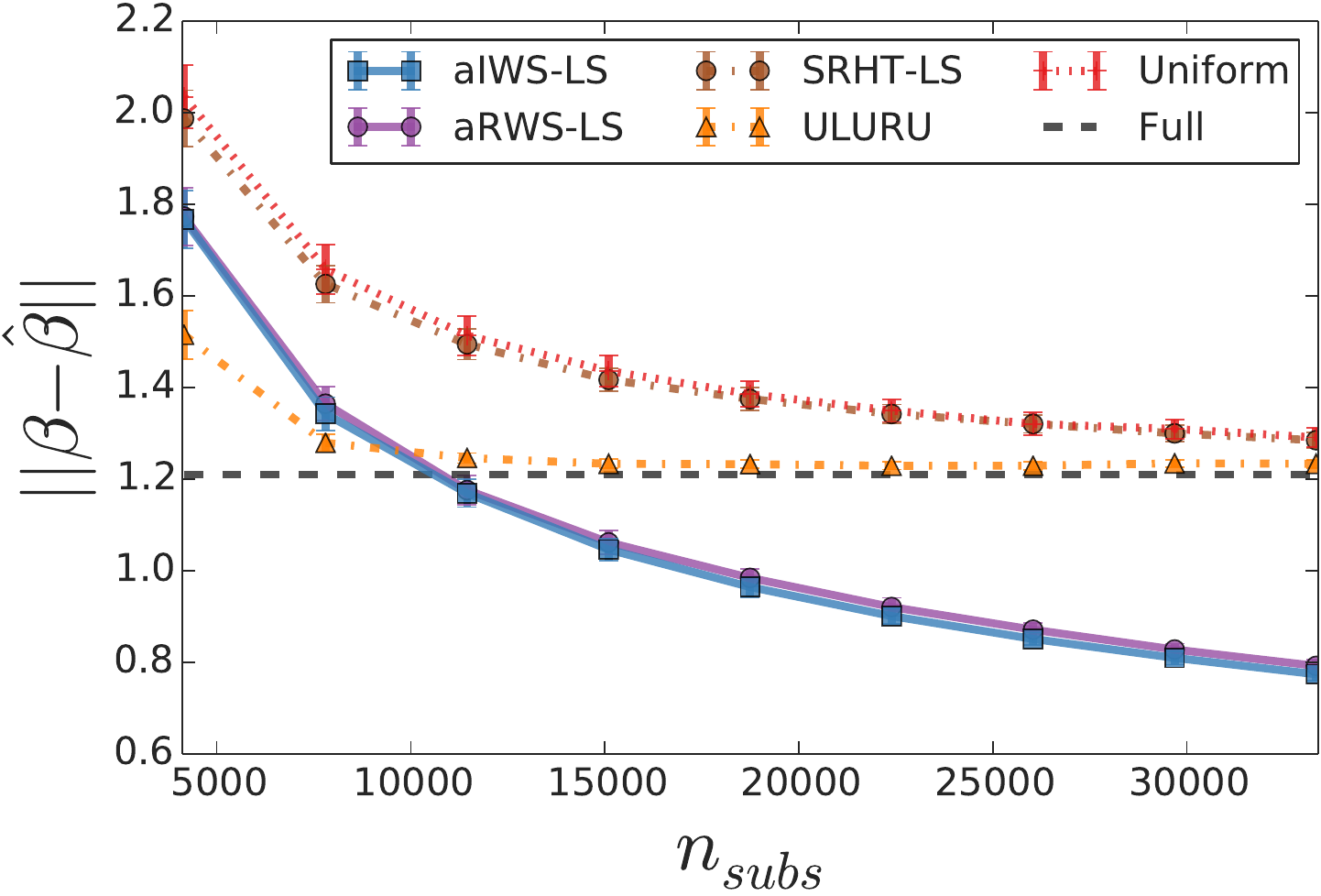}

    \label{fig:c2}
\end{minipage}
}
\subfloat[{Airline delay}]{
\begin{minipage}{0.33\textwidth}
    {\includegraphics[width=0.98\columnwidth]{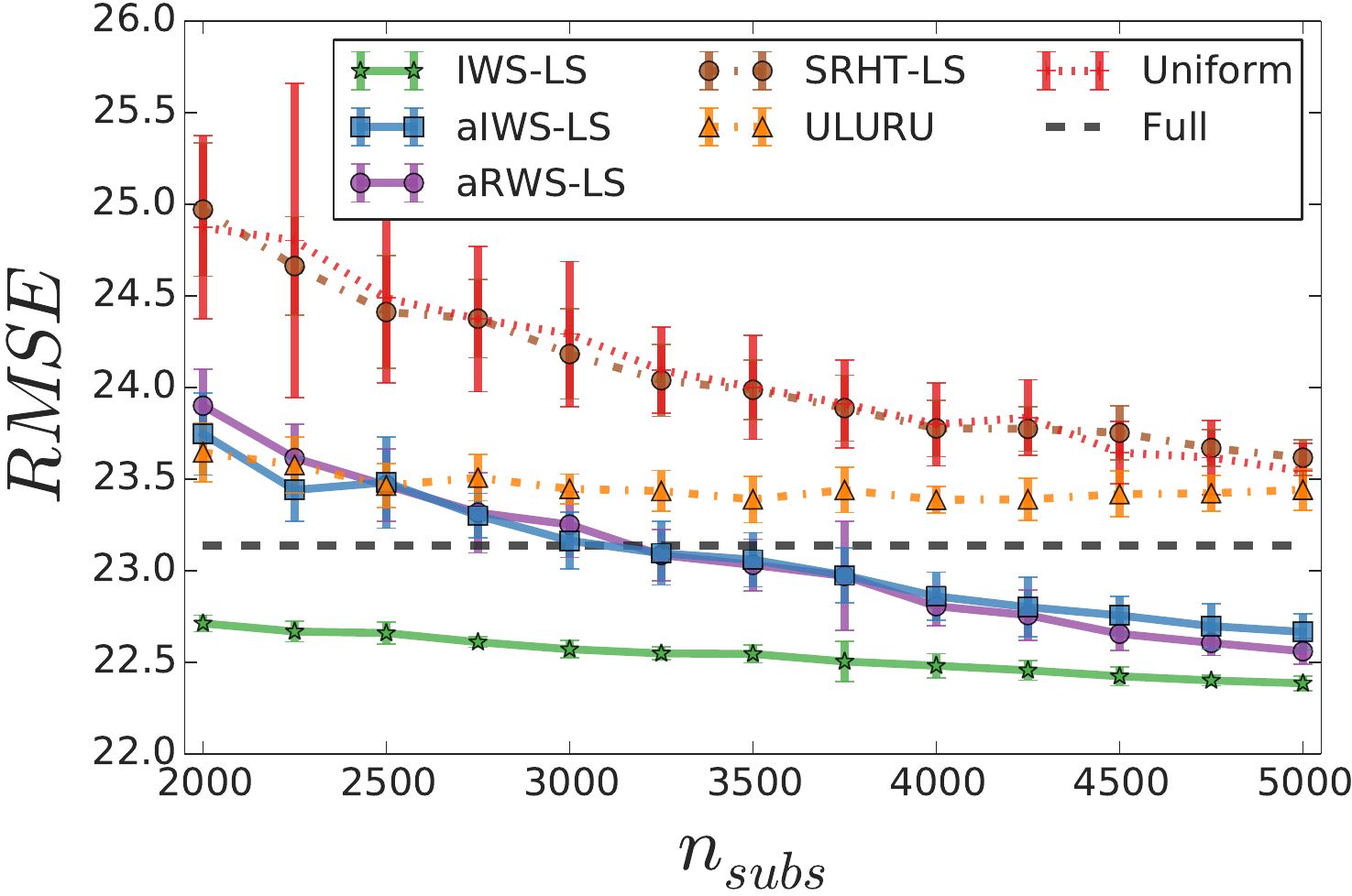} \label{fig:real1}}

\end{minipage}
}
\vspace{-5pt}
\caption{Comparison of mean estimation error and standard deviation on two corrupted simulated datasets and the airline delay dataset. \label{fig:Corrupted}}
\vspace{-5pt}
\end{centering}
\end{figure*}


\paragraph{Airline delay dataset}
\vspace{-5pt}

The dataset consists of details of all commercial flights in the USA over 20
years.\footnote{Dataset along with visualisations available from\\ \url{http://stat-computing.org/dataexpo/2009/}} Selecting the first $\samp_{train}=13,000$ US Airways flights from January 2000 (corresponding to approximately 1.5 weeks) our goal is to predict the
delay time of the next $\samp_{test}=5,000$ US Airways flights.
The features in this dataset consist of a binary vector representing origin-destination pairs and a real value representing distance ($\dims=170$).

The dataset might be expected to violate the usual i.i.d. sub-Gaussian
design assumption of standard linear regression since the length of
delays are often very different depending on the day. For example,
delays may be longer due to public holidays or on weekends. Of course, such regular events could be accounted for in the modelling step, but some unpredictable outliers such as weather delay may also occur.
Results are presented in Figure \ref{fig:Corrupted}\subref{fig:real1}, the RMSE is the error in predicted delay time in minutes. 
 Since the dataset is smaller, we can run \our to observe the accuracy of \ourapprox and \irh in comparison.
 For more than $3000$ samples, these algorithm outperform OLS and quickly approach \our. The result suggests that the corrupted observation model is a good model for this dataset. Furthermore, \uluru  is unable to achieve the full accuracy of the OLS solution.



\section{Conclusions}
We have demonstrated theoretically and empirically under the  generalised corrupted observation model that \iws is able to significantly reduce both the bias and variance compared
with the OLS estimator and other randomized approximations which do
not take influence into account. Importantly our fast approximation,
\irh performs similarly to \our. We find \uluru quickly converges to
the OLS estimate, although it is not able to overcome the bias induced by the corrupted datapoints despite its two-step procedure.  The performance of \our relative to OLS in the airline delay
problem suggests that the corrupted observation model is a more
realistic modelling scenario than the standard sub-Gaussian design
model for some
tasks.




Software is available at {\footnotesize \url{http://people.inf.ethz.ch/kgabriel/software.html}}.
\paragraph{Acknowledgements.} We thank David Balduzzi for invaluable discussions, suggestions and comments.


\bibliography{randomSamplingNips}

\begin{thebibliography}{10}

\bibitem{Ailon:2008}
Nir Ailon and Edo Liberty.
\newblock Fast dimension reduction using rademacher series on dual bch codes.
\newblock In {\em 19th Annual ACM-SIAM Symposium on Discrete Algorithms}, pages
  1--9, 2008.

\bibitem{Belsley:1980}
David~A Belsley, Edwin Kuh, and Roy~E Welsch.
\newblock {\em {Regression Diagnostics}}.
\newblock Identifying Influential Data and Sources of Collinearity. Wiley,
  1981.

\bibitem{Boutsidis:2012tv}
Christos Boutsidis and Alex Gittens.
\newblock {Improved matrix algorithms via the Subsampled Randomized Hadamard
  Transform}.
\newblock 2012.
\newblock arXiv:1204.0062v4 [cs.DS].

\bibitem{Chen:up}
Y~Chen, C~Caramanis, and S~Mannor.
\newblock {Robust Sparse Regression under Adversarial Corruption}.
\newblock In {\em International Conference on Machine Learning}, 2013.

\bibitem{Chen:2012vt}
Yudong Chen and Constantine Caramanis.
\newblock {Orthogonal Matching Pursuit with Noisy and Missing Data: Low and
  High Dimensional Results}.
\newblock June 2012.
\newblock arXiv:1206.0823.

\bibitem{Chen:2013te}
Yudong Chen and Constantine Caramanis.
\newblock {Noisy and Missing Data Regression: Distribution-Oblivious Support
  Recovery}.
\newblock In {\em International Conference on Machine Learning}, 2013.

\bibitem{Dhillon:2013wz}
P~Dhillon, Y~Lu, D~P Foster, and L~Ungar.
\newblock {New Subsampling Algorithms for Fast Least Squares Regression}.
\newblock In {\em Advances in Neural Information Processing Systems}, 2013.

\bibitem{Drineas:2011ts}
Petros Drineas, Malik Magdon-Ismail, Michael~W Mahoney, and David~P Woodruff.
\newblock {Fast approximation of matrix coherence and statistical leverage}.
\newblock September 2011.
\newblock arXiv:1109.3843v2 [cs.DS].

\bibitem{Drineas:2006iw}
Petros Drineas, Michael~W. Mahoney, and S.~Muthukrishnan.
\newblock Sampling algorithms for l2 regression and applications.
\newblock In {\em Proceedings of the Seventeenth Annual ACM-SIAM Symposium on
  Discrete Algorithm}, SODA '06, pages 1127--1136, New York, NY, USA, 2006.
  ACM.

\bibitem{Drineas:2011ih}
Petros Drineas, Michael~W Mahoney, S~Muthukrishnan, and Tam{\'a}s Sarl{\'o}s.
\newblock {Faster least squares approximation}.
\newblock {\em Numerische Mathematik}, 117(2):219--249, 2011.

\bibitem{Hsu:2011va}
Daniel Hsu, Sham Kakade, and Tong Zhang.
\newblock A tail inequality for quadratic forms of subgaussian random vectors.
\newblock {\em Electron. Commun. Probab.}, 17:no. 52, 1--6, 2012.

\bibitem{Loh:2012hf}
Po-Ling Loh and Martin~J Wainwright.
\newblock {High-dimensional regression with noisy and missing data: Provable
  guarantees with nonconvexity}.
\newblock {\em The Annals of Statistics}, 40(3):1637--1664, June 2012.

\bibitem{Ma:2013tx}
Ping Ma, Michael~W Mahoney, and Bin Yu.
\newblock {A Statistical Perspective on Algorithmic Leveraging}.
\newblock In {\em proceedings of the International Conference on Machine
  Learning}, 2014.

\bibitem{Mahoney:2011te}
Michael~W Mahoney.
\newblock {Randomized algorithms for matrices and data}.
\newblock April 2011.
\newblock arXiv:1104.5557v3 [cs.DS].

\bibitem{McWilliams:2012jn}
Brian McWilliams and Giovanni Montana.
\newblock {Multi-view predictive partitioning in high dimensions.}
\newblock {\em Statistical Analysis and Data Mining}, 5(4):304--321, 2012.

\bibitem{McWilliams:2012wi}
Brian McWilliams and Giovanni Montana.
\newblock {Subspace clustering of high-dimensional data: a predictive
  approach}.
\newblock {\em Data Mining and Knowledge Discovery}, 28:736–--772, 2014.

\bibitem{Tropp:2010uo}
Joel~A Tropp.
\newblock {Improved analysis of the subsampled randomized Hadamard transform}.
\newblock November 2010.
\newblock arXiv:1011.1595v4 [math.NA].

\bibitem{Vershynin:2010vka}
Roman Vershynin.
\newblock {Introduction to the non-asymptotic analysis of random matrices}.
\newblock November 2010.
\newblock arXiv:1011.3027.

\end{thebibliography}
\bibliographystyle{plain}

\newpage
\begin{center}
{ \Large{\textbf{Supplementary Information for {\titleName}} }}
\end{center}

\newcounter{si-sec}
\renewcommand{\thesection}{SI.\arabic{si-sec}}

Here we collect supplementary technical details, discussion and empirical results which support the results presented in the main text.

\addtocounter{si-sec}{1}
\section{Software} \label{sec:suppsoftware}
We have made available a software package available for Python which implements
\begin{itemize}
\item \our,
\item \ourapprox and
\item \irh,
\end{itemize}
along with the methods we compare against
\begin{itemize}
\item \srht and
\item \uluru .
\end{itemize}

The software is available from:
{\footnotesize \url{http://bit.ly/1kifYBU}}

\addtocounter{si-sec}{1}
\section{Approximate Influence Weighted Algorithm} \label{sec:suppalg}

Here we present a detailed description of the approximate \iws (\ourapprox) algorithm. Steps 2, 3 and 4 are required for the approximate leverage computation. Step 3 could be replaced with the QR decomposition.
\begin{algorithm}[ht]
\caption{Approximate influence weighted subsampling (\ourapprox).\label{alg:dsampapprox}}
	\algorithmicrequire\; Data: $\Zt$, $\y$
  \begin{algorithmic}[1]
	\STATE {\bf \emph{Solve $\estbeta_{SRHT} = \arg \min_{\boldbeta} \nrm{\boldPi_1\cdot\y - \boldPi_1\cdot\Zt\boldbeta}^2$}}
	\STATE {\bf \emph{SVD: $(\Ut,\boldSigma,\Vt) = \boldPi_1\cdot\Zt$}} \COMMENT{Compute basis for randomized leverage approximation.}
    \STATE {\bf \emph{$\Rt^{-1} = \Vt\boldSigmai$}}
    \STATE {\bf \emph{$\tilde{\Ut} = \Zt\Rt^{-1} \cdot \boldPi_2$}}
    \FOR{$i=1\ldots\samp$}
               \STATE $\tilde{l}_i = \nrm{\tilde{\Ut}_i}$
           \STATE $\tilde{e}_i = y_i - \z_i\estbeta_{SRHT}$ 
           \STATE $\tilde{d}_i = \tilde{e}_i^2 \tilde{l}_i/(1-\tilde{l}_i)^2$
    \ENDFOR
    \STATE {\bf \emph{Sample rows ($\tilde{\Zt}$, $\tilde{\y}$) of ($\Zt$, $\y$) proportional to $\frac{1}{\tilde{d}_i}$}}
    \STATE {\bf \emph{Solve}} $\estbeta_{aIWS} = \arg\min_{\boldbeta} \nrm{\tilde{\y} - \tilde{\Zt}\boldbeta}^2$
  \end{algorithmic}
  \algorithmicensure\; $\estbeta_{aIWS}$
\end{algorithm}

\addtocounter{si-sec}{1}
\section{Leverage and Influence} \label{sec:supp_lev}
Here we provide detailed derivations of leverage and influence terms as well as the full statement and proofs of finite sample bounds under the sub-Gaussian design and corrupted design models which are abbreviated in the main text as Lemmas \ref{prop:lev1}, \ref{prop:infl1},  \ref{prop:lev2},  and \ref{prop:infl2}.

Here we provide a full derivation of the leave-one-out estimator of $\estbeta$ which appears in less detail in \cite{Belsley:1980}.
\begin{prop}[Derivation of $\estbeta_{-i}$ \label{prop:betai}]
Defining $e_i = \hat{y}_i - y_i$ and $\boldSigma = \Xt\tr\Xt$
\begin{align*}
\estbeta_{-i} &  = \br{\boldSigma - \x_i\tr\x_i}^{-1}\br{\Xt\tr\y - \x_i\tr y_i}
\\
& = \br{\boldSigmai + \frac{\boldSigmai\x_i\x_i\tr\boldSigmai}{1-l_i}}\br{\Xt\tr\y - \x_i\tr y_i}
\\
& = \estbeta - \boldSigmai\x_i\tr\br{ y_i + \frac{\x_i\boldSigmai\Xt\y - \x_i\boldSigmai\x_i\tr y_i}{1-l_i}}
\\
& = \estbeta - \boldSigmai\x_i\tr\br{ y_i + \frac{\hat{y}_i - l_i y_i}{1-l_i}}
\\
& =  \estbeta - \boldSigmai\x_i\tr\br{ y_i + \frac{e_i}{1-l_i} - \frac{y_i(1-l_i)}{1-l_i} }
\\
& = \estbeta - \frac{\boldSigmai\x_i\tr e_i}{1-l_i}
\end{align*}
Where the first equality comes from a straightforward application of the Sherman Morrison formula.
\end{prop}

Here we provide a derivation of the leave-one-out estimator in the corrupted model where the point we removed is corrupted.
\begin{prop}[Derivation of $\estbeta_{-m}$] \label{prop:betan}
By proposition \ref{prop:betai}. Defining
$$e_m = \hat{y}_m - y_m = (\x_m+\wt_m)\estbeta - y_m ~~ \text{ and }$$
$$l_m = (\x_m+\wt_m)\boldSigmai(\x_m+\wt_m)\tr$$
where $\boldSigma = \Zt\tr\Zt$,
we have that
\begin{align*}
\estbeta_{-m} = \estbeta -\frac{ \boldSigmai\br{\x_m + \wt_m}\tr e_m}{1-l_m} .
\end{align*}
\end{prop}

\subsection{Results for Sub-Gaussian random design}

\begin{lem}[Leverage]\label{prop:lev1} The leverage of a non-corrupted point is bounded by
\begin{equation} \label{eq:lev}
l_i  \leq \sigma_x^2 \cdot \order{(\dims/\sqrt{\samp})^2}
\end{equation}
where the exact form of the $\order{(\dims/\sqrt{\samp})^2}$ term is given in the supplementary material.
\end{lem}

\begin{lem}[Influence]\label{prop:infl1}
Defining $E := \nrm{\estbeta_{\OLS} - \boldbeta}$, the influence of a non-corrupted point is
\begin{equation}\label{eq:inf}
d_i \leq C_i \br{ \sigma_x\sigma_{\epsilon} + \sigma^2_x E}.
\end{equation}
The $C_i$ term is proportional to $\log\dims \sqrt{\dims}\nrm{\boldSigmai}/ (1-l_i)$.
\end{lem}

\paragraph{Proof of Lemma \ref{prop:lev1}.}
Lemma \ref{prop:lev1} states
$$
l_i  \leq \sigma_x^2 \cdot \br{ \frac{\dims + 2\log \dims + 2\sqrt{\dims\log \dims }}{\sqrt{\samp} - C\sqrt{\dims}- \sqrt{\log \dims}}}^2 .
$$
From the Eigen-decomposition, $\boldSigma = \Vt\boldLambda\Vt\tr$. Define $\At = \boldLambda^{-1/2}\Vt$ such that $\At\tr\At = \boldSigmai$. We have
\begin{align*}
l_i & = \x_i\boldSigmai\x_i\tr \\
& = \nrm{\At \x_i \tr}^2
\end{align*}
Since $\x$ and $\wt$ are sub-Gaussian random vectors so the above quadratic form is bounded by Lemma \ref{lem:hsu}, setting the parameter $t=\log \dims$. We combine this with the following inequalities
$$
\sqrt{\trace{\boldSigma^{-2}}} = \nrm{\boldSigmai}_{F} \leq \sqrt{\dims}\nrm{\boldSigmai} = \sqrt{\dims}\sigma_1(\At)^2
$$
and
$$
\trace{\boldSigmai} = \nrm{\At}_F^2 \leq \br{\sqrt{\dims}\nrm{\At}}^2 = \dims\sigma_1(\At)^2
$$
which relate the Frobenius norm with the spectral norm.  We also make use of the relationship $\sigma_n(\Zt)^{-1} = \sigma_1(\At)$ where $\Zt = \Xt + \Wt$
to obtain
\begin{align*}
\nrm{\At\x}^2 & \leq \sigma_x^2 \sigma_{\samp}(\Zt)^{-2} \br{\dims + 2\log \dims + 2\sqrt{\dims\log\dims}}
\end{align*}
which holds with high probability.

In order for this bound to not be vacuous in our application, it must be smaller than 1. In order to ensure this, we need bound $\sigma_n(\Zt)^{-1}$ using Lemma \ref{lem:sval} and setting $\tau=\sqrt{c_0 \log \dims}$  to obtain the following which holds with high probability
\begin{align*}
\nrm{\At\x}^2  
& \leq \sigma_x^2 \br{\frac{  \dims + 2\log \dims + 2\sqrt{\dims\log \dims} }{\sqrt{\samp} - C\sqrt{\dims}- \tau}}^2
\\
& \leq \sigma_x^2 \br{\frac{ \dims + 2\log \dims + 2\sqrt{\dims\log \dims} }{\sqrt{\samp} - C\sqrt{\dims}- \sqrt{\log \dims}}}^2 .
\end{align*}
\qed

\paragraph{Proof of Lemma \ref{prop:infl1}.}
Defining 
$
\boldSigma  = \Zt\tr\Zt 
$
, Lemma \ref{prop:infl1} states
\begin{align*}
\nrm{ \estbeta_{-i} - \estbeta}
\leq & \frac{\nrm{\boldSigmai}}{1-l_i} \br{ \sigma_x\sigma_{\epsilon} + 2\sigma^2_x\nrm{\boldbeta - \estbeta}}\sqrt{\dims} \log\dims .
\end{align*}
Using Proposition \ref{prop:betai} we have
\begin{align*}
\nrm{ \estbeta_{-i} - \estbeta} = & \frac{1}{1-l_i}\nrm{\boldSigmai\x_i\tr e_i}
\\
= & \frac{1}{1-l_i} \nrm{\boldSigmai \x_i\tr\br{\epsilon + \x_i\br{\boldbeta - \estbeta}}}
\\
\leq & \frac{1}{1-l_i} \nrm{\boldSigmai}\nrm{\x_i\tr\epsilon + \x_i\tr\x_i\br{\boldbeta - \estbeta}}
\\
\leq & \frac{1}{1-l_i} \nrm{\boldSigmai}\br{\nrm{\x_i\tr\epsilon} + \nrm{\x_i\tr\x_i (\boldbeta - \estbeta)}} .
\end{align*}
Using Corollary \ref{lem:nequals1} to bound $\nrm{\x_i\tr\epsilon}$ and  $\nrm{\x_i\tr\x_i (\boldbeta - \estbeta)}$ (since for these terms $\samp=1$ and so Lemma \ref{lem:Chen1} does not immediately apply) completes the proof.
\qed


\subsection{Results for corrupted observations}

\begin{lem}[Leverage of corrupted point] \label{prop:lev2}
The leverage of a corrupted point is bounded by
\begin{equation} \label{eq:levNoise}
l_m \leq (\sigma_x^2 + \sigma_w^2) \cdot \order{(\dims/\sqrt{\samp})^2} .
\end{equation}
\end{lem}

\begin{rem}[Comparison of leverage] \label{rem:comp_lev}
Comparing this with Eq. \eqref{eq:lev}, when $\samp$ is large, the dominant term is $O((\dims /\sqrt{ \samp})^2)$ which implies that the difference in leverage between a corrupted and non-corrupted point -- particularly when the magnitude of corruptions is not large -- is small. This suggests that it may not be possible to distinguish between the corrupted and non-corrupted points by only comparing leverage scores.
\end{rem}

\begin{lem}[Influence of corrupted point] \label{prop:infl2}
Defining $E := \nrm{\estbeta_{\OLS} - \boldbeta}$, the influence of a corrupted point is
\begin{align}
d_m \leq & C_m
 (\sigma_x\sigma_w + \sigma_w^2)\nrm{\boldbeta}
 + (\sigma_x^2+ 2\sigma_x\sigma_w + \sigma_w^2)E  \nonumber
 \\
 &  + (\sigma_x+\sigma_w)\sigma_{\epsilon} \label{eq:infNoise} .
\end{align}
\end{lem}
\begin{rem}[Comparison of influence] \label{rem:comp_inf}
Here, $C_m$ differs from $C_i$ in Lemma \ref{prop:infl1} only in its dependence on the leverage of a corrupted instead of non-corrupted point and so for large $\samp$, $C_i\approx C_m$. It can be seen that the influence of the corrupted point includes a bias term similar to the one which appears in Eq. \eqref{eq:thm4}. This suggests that the relative difference between the influence of a non-corrupted and corrupted point will be larger than the respective relative difference in leverage. All of the information relating to the proportion of corrupted points is contained within $E$.
\end{rem}


\paragraph{Proof of Lemma \ref{prop:lev2}.}
Lemma \ref{prop:lev2} states
\begin{align*}
l_m \leq (\sigma_x^2+\sigma_w^2) \cdot \br{\frac{\dims + 2\log\dims + 2\sqrt{\dims\log\dims}}{\sqrt{\samp} - C\sqrt{\dims}-
\sqrt{\log\dims}}}^2 .
\end{align*}
The proof follows from rewriting $l_m =   \nrm{\At (\x_m+\wt_m)\tr}^2$ and following the same steps as the proof of Lemma \ref{prop:lev1} above.
\qed

\paragraph{Proof of Lemma \ref{prop:infl2}.}
Lemma \ref{prop:infl2} states
\begin{align*}
\nrm{ \estbeta_{-m} - \estbeta}
 \leq & \frac{\nrm{ \boldSigmai} }{1-l_m}
 \biggl(
 2(\sigma_x\sigma_w + \sigma_w^2)\nrm{\boldbeta}
 + 2(\sigma_x^2+ \sigma_x\sigma_w + \sigma_w^2)  \cdot \nrm{\boldbeta-\estbeta}
 + 2(\sigma_x+\sigma_w)\sigma_\epsilon
 \biggr)
 \\
 &
 \cdot \sqrt{\dims} \log\dims  .
\end{align*}

From Proposition \ref{prop:betan} and following the same argument as Lemma \ref{prop:infl1} we have
\begin{align*}
\nrm{ \estbeta_{-m} - \estbeta} = & \frac{1}{1-l_m} \nrm{ \boldSigmai (\x_m + \wt_m)\tr e_m }
\\
\leq & \frac{1}{1-l_m} \Vert \boldSigmai (\x_m+\wt_m)\tr \biggl((\x_m+\wt_m)(\boldbeta - \estbeta)
 + \wt_m\boldbeta + \epsilon\biggr)  \Vert
\\
\leq & \frac{1}{1-l_m}\nrm{ \boldSigmai}\biggl( \nrm{\x_m \tr \wt_m\boldbeta} + \nrm{\wt_m \tr \wt_m\boldbeta} + \nrm{\x_m\tr \epsilon} + \nrm{\wt_m\tr \epsilon} 
\\
&  + \nrm{\br{\x_m\tr \x_m + \wt_m\tr \wt_m +  2\x_m\tr \wt_m}(\boldbeta - \estbeta)}  \biggr) .
\end{align*}
Applying the triangle inequality followed by Corollary \ref{lem:nequals1} and noting that $(\sigma_x\sigma_w + 2\sigma_w^2)\leq2(\sigma_x\sigma_w + \sigma_w^2)$  completes the proof.
\qed


\addtocounter{si-sec}{1}
\section{Estimation error in sub-Gaussian model} \label{sec:subgthm}
Using the definition of influence above, we can state the following theorem characterising the error of the \iws estimator in the sub-Gaussian design setting.
\begin{thm}[Sub-gaussian design \iws] \label{thm:standard}
Defining $E = \nrm{\estbeta_{\OLS} - \boldbeta}$ for $\samp\gtrsim \frac{\sigma_x^2}{\lambda_{\min}(\Sigma_{\Theta x})}\dims\log \dims$ we have
\begin{align*}
\nrm{\estbetaSS - \boldbeta}
 \lesssim & \frac{1}{\lambda} \cdot  \frac{\sigma_\epsilon  }{\lambda_{\min}(\Sigma_{x}) (\sigma_\epsilon + 2\sigma_xE)} \cdot \sqrt{\frac{1}{r\samp}}
\end{align*}
  where
$0\leq \lambda \leq \lambda_{\min}(\Sigma_{\Theta x})$ and
$\Sigma_{\Theta x}$ is the covariance of the influence weighted subsampled data and $r=\samp_{subs}/\samp$.
\end{thm}
\gabriel{Where is $\lambda_1(\Sigma_x)$ from the proof? Don't we also need $\lambda$ to be smaller than $\lambda_{\min}(\Sigma_{x}$?}
\begin{rem}
Theorem \ref{thm:standard} states that in the non-corrupted sub-Gaussian model, the \iws estimator is consistent. Furthermore,
if we set the sampling proportion, $r\geq \order{1/\dims}$, the error scales as $\order{\sqrt{\dims/\samp}}$. Therefore, similar to \uluru there is no dependence on the subsampling proportion.
\end{rem}

\addtocounter{si-sec}{1}
\section{Proof of main theorems}
In this section we provide proofs of our main theorems which describe the properties of the \iws estimator in the sub-Gaussian random design case, the OLS estimator in the corrupted setting and finally our \iws estimator in the corrupted setting.

\vspace{0.5cm}
In order to prove our results we require the following lemma
\begin{lem}[A general bound on $\nrm{\estbeta - \boldbeta}$ from \cite{Chen:2013te}] \label{lem:chenLSbound} Suppose the following strong convexity condition holds: $\lambda_{\min}(\widehat{\boldSigma})\geq\lambda > 0$. Then the estimation error satisfies
$$
\nrm{\estbeta - \boldbeta} \lesssim \frac{1}{\lambda} \nrm{\hat{\gamma} - \widehat{\boldSigma}\boldbeta} .
$$
Where $\hat{\gamma}$, $\widehat{\boldSigma}$ are estimators for $\bE\sq{{\Xt\tr\y}}$ and $\bE\sq{\Xt\tr\Xt}$ respectively
\end{lem}
To obtain the results for our method in the non-corrupted and corrupted setting we can simply plug in our specific estimates for $\hat{\gamma} $ and $\widehat{\boldSigma}$.

\paragraph{Proof of Theorem \ref{thm:standard}.}
Through subsampling according to influence, we solve the problem
$$
\estbetaSS = \arg\min_{\boldbeta} \nrm{ \Theta \y - \Theta \Xt\boldbeta}^2
$$
where $\Theta = \sqrt{\frac{\samp}{\samp_{subs}}} \St\Dt$. $\St$ is a subsampling matrix, $\Dt$ is a diagonal matrix whose entries are $\sqrt{p_i/\samp} = \sqrt{c / d_i \samp}$ where $c$ is a constant which ensures $\sum_{i=1}^n p_i = 1$.
\begin{equation} \label{eq:Di}
D_{ii}^2 \propto \br{\frac{\nrm{\boldSigmai}}{1-l_i} \br{ \sigma_x\sigma_{\epsilon} + 2\sigma^2_x\nrm{\boldbeta - \estbeta}}\sqrt{\dims} \log\dims }^{-1} .
\end{equation}

Setting $\hat{\gamma} = (\Theta \Xt)\tr\y$, $\widehat{\boldSigma} = (\Theta \Xt)\tr(\Theta \Xt)$, by Lemma \ref{lem:chenLSbound} the error of the \iws estimator is given by
\begin{align}
\frac{1}{\lambda} \nrm{\hat{\gamma} - \widehat{\boldSigma}\boldbeta} = & \nrm{(\Theta \Xt)\tr(\Theta \y) - (\Theta \Xt)\tr(\Theta \Xt)\boldbeta}
\nonumber\\
 = &\frac{1}{\lambda} \nrm{(\Theta \Xt)\tr(\Theta \epsilon) + (\Theta \Xt)\tr(\Theta \Xt)\boldbeta - (\Theta \Xt)\tr(\Theta \Xt)\boldbeta}
\nonumber\\
 = & \frac{1}{\lambda} \nrm{(\Theta \Xt)\tr(\Theta \epsilon)}
 \label{thm:thm3bound}
\end{align}


Now, by Lemma \ref{lem:Chen1} we have
$$
\nrm{\Xt \tr \epsilon} \leq \sigma_x\sigma_\epsilon\sqrt{\frac{\dims\log\dims}{\samp}}
$$
and so defining $E = \nrm{\boldbeta - \estbeta}$,
\begin{align}
\nrm{(\Theta \Xt) \tr \Theta\epsilon} \leq & \nrm{\frac{1}{r\samp}\sum_{i=1}^{r\samp} p_i }\cdot\nrm{(\St \Xt) \tr \St\epsilon}
\nonumber \\
\leq &\nrm{\frac{1}{r\samp}\sum_{i=1}^{r\samp} (1 - l_i)} \frac{\sigma_x\sigma_\epsilon\sqrt{\dims\log\dims/r\samp}}{\nrm{\boldSigmai}\br{ \sigma_x\sigma_{\epsilon} + 2\sigma^2_xE}\sqrt{\dims}\log\dims }
\nonumber \\
\leq &\frac{\sigma_x\sigma_\epsilon\sqrt{\dims\log\dims/r\samp}}{\nrm{\boldSigmai}\br{ \sigma_x\sigma_{\epsilon} + 2\sigma^2_xE} \sqrt{\dims}\log\dims }
\nonumber \\
\leq  & \frac{ \sigma_\epsilon\sqrt{1/r\samp}}{ \lambda_{\min}(\Sigma_x) (\sigma_\epsilon + 2 \sigma_x E)}
\label{thm:thm3bound1}
\end{align}
where the third inequaltiy uses the fact that $\sum_{i=1}^\samp (1-l_i) \leq \samp$.

%
%

Define $\Sigma_{\Theta x} = \bE\sq{(\Theta\Xt)\tr(\Theta\Xt)}$. Now, when $n\gtrsim \frac{(\sigma_x^2) \dims\log\dims}{\lambda_{\min}(\Sigma_{\Theta x})}$ using Lemma \ref{lem:inverseCov} with $\lambda = \lambda_{\min}(\Sigma_{\Theta x})$
we have w.h.p. $\lambda_1((\Theta\Xt)\tr(\Theta\Xt)  - \Sigma_{\Theta x})\leq \frac{1}{54}\lambda_{\min}(\Sigma_{\Theta x})$. It follows that
\begin{align}
\lambda_{\min}((\Theta\Xt)\tr(\Theta\Xt))  &=  \inf_{\nrm{\vt}=1}\vt\tr \br{\Sigma_{\Theta x} + (\Theta\Xt)\tr(\Theta\Xt) - \Sigma_{\Theta x}) }\vt
\nonumber\\
& \geq  \lambda_{\min}(\Sigma_{\Theta x}) - \lambda_1((\Theta\Xt)\tr(\Theta\Xt) - \Sigma_{\Theta x}))
\nonumber\\
& \geq \frac{1}{2}\lambda_{\min}(\Sigma_{\Theta x})  .
\label{thm:thm3bound2}
\end{align}
Using \eqref{thm:thm3bound1} and \eqref{thm:thm3bound1} in Eq. \eqref{thm:thm3bound} completes the proof.

\qed

\vspace{0.5cm}

\begin{rem}[Scaling by $\pi$] \label{rem:pi} In the following, with some abuse of notation we will write $U\Wt$ as $\Wt$.
Now,
\begin{align*}
\nrm{\Wt} & := \nrm{U \Wt}
\\
&\leq \pi  \nrm{\Wt} .
\end{align*}
\end{rem}

\vspace{0.5cm}

\paragraph{Proof of Theorem \ref{thm:corruptedls}.}
Setting $\hat{\gamma} = \Zt\tr\y$, $\widehat{\boldSigma} = \Zt\tr\Zt$ we have
\begin{align*}
\nrm{\hat{\gamma} - \widehat{\boldSigma}\boldbeta} & = \nrm{(\Xt + \Wt)\tr\y - (\Xt+\Wt)\tr(\Xt+\Wt)\boldbeta}
\\
& = \nrm{\Xt\tr(\Xt\boldbeta + \epsilon) + \Wt\tr(\Xt\boldbeta +
  \epsilon) -\Xt\tr\Xt\boldbeta -\Wt\tr\Wt\boldbeta -
  \Xt\tr\Wt\boldbeta - \Wt\tr\Xt\boldbeta
}
\\
& = \nrm{ \Xt\tr\epsilon + \Wt\tr\epsilon - \Xt\tr\Wt\boldbeta - \Wt\tr\Wt\boldbeta}
\\
& \leq  \nrm{\Xt\tr\epsilon} + \nrm{\Wt\tr\epsilon} + \nrm{\Xt\tr\Wt\boldbeta} + \nrm{\Wt\tr\Wt\boldbeta} .
\end{align*}

From Lemma \ref{lem:Chen1} and Remark \ref{rem:pi} we have w.h.p.
\begin{align}
\nrm{\Xt\tr\epsilon} & \leq \sigma_x\sigma_\epsilon  \sqrt{\frac{\dims \log \dims}{\samp}} \label{eq:xte}
\\
\nrm{\Wt\tr\epsilon} & \leq   \pi  \sigma_w\sigma_\epsilon  \sqrt{\frac{\dims \log \dims}{\samp}} \label{eq:wte}
\\
\nrm{\Xt\tr\Wt \boldbeta} & \leq  \pi  \sigma_x\sigma_w \nrm{\boldbeta} \sqrt{\frac{\dims \log \dims}{\samp}} \label{eq:xtwb}
\\
\nonumber\\
\nrm{\Wt\tr\Wt \boldbeta} & = \nrm{\br{\Wt\tr\Wt + \sigma_w^2\It_{\dims} - \sigma_w^2\It_{\dims}} \boldbeta} \label{eq:wtwb}
\nonumber \\
& \leq \nrm{\br{\Wt\tr\Wt - \sigma_w^2\It_{\dims}} \boldbeta} + \sigma_w^2\nrm{\boldbeta}
\nonumber \\
& \leq \pi  \sigma_w^2\br{C\sqrt{\frac{\dims \log \dims}{\samp}} +\sqrt{\dims} }\nrm{\boldbeta} .
\end{align}

Now, when $n\gtrsim \frac{(\sigma_x^2\sigma^2_w) \dims\log\dims}{\lambda_{\min}(\Sigma_x)}$ using Lemma \ref{lem:inverseCov} with $\lambda = \lambda_{\min}(\Sigma_x)$
we have w.h.p. $\lambda_1(\Zt\tr\Zt  - (\Sigma_x+\Sigma_w))\leq \frac{1}{54}\lambda_{\min}(\Sigma_x)$. It follows that
\begin{align*}
\lambda_{\min}(\Zt\tr \Zt)  &=  \inf_{\nrm{\vt}=1}\vt\tr \br{ \Sigma_x+\Sigma_w+ \Zt\tr\Zt - (\Sigma_x+\Sigma_w) }\vt
\\
& \geq \lambda_{\min}(\Sigma_x) + \lambda_{\min}(\Sigma_w) - \lambda_1(\Zt\tr\Zt  - (\Sigma_x+\Sigma_w))
\\
& \geq \frac{1}{2}\lambda_{\min}(\Sigma_x) + \pi\lambda_{\min}(\Sigma_w)  .
\end{align*}

Using Lemma \ref{lem:chenLSbound} with Eqs. (\ref{eq:xte}-\ref{eq:wtwb}) and the above bound for $\lambda=\lambda_{\min}(\Zt\tr\Zt)$ completes the proof. \qed


\paragraph{Proof of Theorem \ref{thm:corrupt}.}
when $n\gtrsim \frac{(\sigma_x^2\sigma^2_w) \dims\log\dims}{\lambda_{\min}(\Sigma_{\Theta x})}$ using Lemma \ref{lem:inverseCov} with $\lambda = \lambda_{\min}(\Sigma_{\Theta x})$
we have w.h.p. $\lambda_1((\Theta\Zt)\tr(\Theta\Zt)  - \Sigma_{\Theta x})\leq \frac{1}{54}\lambda_{\min}(\Sigma_{\Theta x})$. It follows that
\begin{align*}
\lambda_{\min}((\Theta\Zt)\tr(\Theta\Zt))  &=  \inf_{\nrm{\vt}=1}\vt\tr \br{\Sigma_{\Theta x} + (\Theta\Zt)\tr(\Theta\Zt) - \Sigma_{\Theta x}) }\vt
\\
& \geq  \lambda_{\min}(\Sigma_{\Theta x}) - \lambda_1((\Theta\Zt)\tr(\Theta\Zt) - \Sigma_{\Theta x}))
\\
& \geq \frac{1}{2}\lambda_{\min}(\Sigma_{\Theta x})  .
\end{align*}

From the bound in Lemma \ref{lem:chenLSbound} we have
\begin{align*}
\nrm{\hat{\gamma} - \widehat{\boldSigma}\boldbeta}
& \leq  \nrm{\br{\Theta \Xt}\tr \br{\Theta\epsilon}} + \nrm{\br{\Theta \Wt}\tr\br{\Theta \epsilon}}
\\
& + \nrm{\br{\Theta\Xt}\tr\br{\Theta\Wt}\boldbeta} + \nrm{\br{\Theta \Wt}\tr \br{\Theta\Wt}\boldbeta} .
\end{align*}
We now aim to show that the relative contribution of the corrupted points is decreased under the \iws scheme. To show this, we first multiply both corrupted and non-corrupted points by
$$
\nrm{\boldSigmai}\br{ \sigma_x\sigma_{\epsilon} + 2\sigma^2_x\nrm{\boldbeta - \estbeta}}\log\dims\sqrt{\dims} .
$$
This is equivalent to multiplying the non-corrupted points by the subsampling matrix $\St$ and scaling and subsampling the corrupted points by the following term $\Theta_M = \sqrt{\frac{\samp}{\samp_{subs}}}\St \Dt_M$ where $\Dt_M$ has squared diagonal entries proportional to
\begin{align*}
D_M^2 \propto \frac{1}{\samp} \cdot \frac{ \sigma_\epsilon\sigma_x + 2 \sigma_x^2 E }{ 2(\sigma_w^2 + \sigma_w\sigma_x)\nrm{\boldbeta} + 2(\sigma_w^2 + \sigma_w\sigma_x + \sigma_x^2)E + 2(\sigma_w + \sigma_x)\sigma_{\epsilon}} .
\end{align*}
Now we have
\begin{align*}
\nrm{\hat{\gamma} - \widehat{\boldSigma}\boldbeta}
& \lesssim  \nrm{\br{\St \Xt}\tr \br{\St\epsilon}} + \nrm{\br{\Theta_M \Wt}\tr\br{\Theta_M \epsilon}}
\\
& + \nrm{\br{\Theta_M\Xt}\tr\br{\Theta_M\Wt}\boldbeta} + \nrm{\br{\Theta_M \Wt}\tr \br{\Theta_M\Wt}\boldbeta} .
\end{align*}

Applying Lemma \ref{lem:Chen1} we have w.h.p.
\begin{align}
\nrm{(\St\Xt)\tr(\St\epsilon)} & \lesssim \sigma_x\sigma_\epsilon  \sqrt{\frac{\dims \log \dims}{r\samp}} 
\\
\nrm{\br{\Theta_M \Wt}\tr\br{\Theta_M \epsilon}} & \lesssim \frac{\pi \cdot (\sigma_\epsilon + 2 E)    \pi  \sigma_w\sigma_\epsilon  \sqrt{\frac{\dims \log \dims}{r\samp}}}{ 2(\sigma_w^2 + \sigma_w\sigma_x)\nrm{\boldbeta} + 2(\sigma_w^2 + \sigma_w\sigma_x + \sigma_x^2)E + 2(\sigma_w + \sigma_x)\sigma_{\epsilon}} \label{eq:wteIWS}
\\
\nrm{\br{\Theta_M\Xt}\tr\br{\Theta_M\Wt}\boldbeta} & \lesssim \frac{\pi \cdot (\sigma_\epsilon + 2 E)   \sigma_x\sigma_w \nrm{\boldbeta} \sqrt{\frac{\dims \log \dims}{r\samp}}}{ 2(\sigma_w^2 + \sigma_w\sigma_x)\nrm{\boldbeta} + 2(\sigma_w^2 + \sigma_w\sigma_x + \sigma_x^2)E + 2(\sigma_w + \sigma_x)\sigma_{\epsilon}} \label{eq:xtwbIWS}
\\
\nonumber\\
\nrm{(\Theta_M\Wt)\tr(\Theta_M\Wt) \boldbeta} & \lesssim \frac{\pi \cdot (\sigma_\epsilon + 2 E)  \sigma_w^2\br{C\sqrt{\frac{\dims \log \dims}{r\samp}} +\sqrt{\dims} }\nrm{\boldbeta}}{ 2(\sigma_w^2 + \sigma_w\sigma_x)\nrm{\boldbeta} + 2(\sigma_w^2 + \sigma_w + 1)E + 2(\sigma_w + 1)\sigma_{\epsilon}} . \label{eq:wtwbIWS}
\end{align}

We observe that each of the quantities in Eqs. (\ref{eq:wteIWS} - \ref{eq:wtwbIWS}) are scaled by a term proportional to
\begin{align}
\frac{\pi \cdot (\sigma_\epsilon\sigma_x + 2 \sigma_x^2 E) }{ 2(\sigma_w^2 + \sigma_w\sigma_x)\nrm{\boldbeta} + 2(\sigma_w^2 + \sigma_w\sigma_x + \sigma_x^2)E + 2(\sigma_w + \sigma_x)\sigma_{\epsilon}} . \label{eq:scalefactor}
\end{align}

Taking the limit of large $E$ of the above (see remark \ref{rem:limit}) and setting $\sigma_x=1$ we get
\begin{align*}
\pi^* = \lim_{E\rightarrow \infty} 
& =  \frac{\pi}{\br{\sigma_w^2 + \sigma_w}} .
\end{align*}
Replacing the scaling factor in Eq. \eqref{eq:scalefactor} with $\pi^*$ completes the proof.
\qed
\begin{rem}[Taking $\lim_{\nrm{\estbeta_{\OLS} - \boldbeta} \rightarrow \infty}$] \label{rem:limit}
Intuitively, when $E=\nrm{\estbeta_{\OLS} - \boldbeta}$ is small, this suggests that the effect of the corruptions is negligible and the full (or subsampled) least squares solution is close to optimal. Alternatively, when $E$ is large, the corruptions have a large effect on the estimate and so influence subsampling should work well. Note that here the size of $E$ is dependent on $\sigma_w$ and $\pi$. If we send $E\rightarrow \infty$ by allowing many points to be corrupted, the relative performance of \our compared with OLS worsens. However if we allow $\sigma_w$ to be large, the relative performance of our method improves.
\end{rem}

\vspace{2cm}

%
%

\clearpage

\addtocounter{si-sec}{1}
\section{Supporting concentration inequalities}
Here we collect results which are useful in the statements and proofs of our main theorems. Aside from Corrolary \ref{lem:nequals1} which is a simple modification of Lemma \ref{lem:Chen1}, we defer the proofs to their original papers.


\vspace{0.5cm}

\begin{lem}[Originally Lemma 25 from \cite{Chen:2012vt}] \label{lem:Chen1}
Suppose $\Xt\in\R^{\samp\times k}$ and $\Wt\in\R^{\samp \times m}$ are zero-mean sub-Gaussian matrices with parameters $(\frac{1}{n}\Sigma_x,\frac{1}{n}\sigma^2_x)$, $(\frac{1}{n} \Sigma_w,\frac{1}{n}\sigma^2_w)$ respectively. Then for any fixed vectors $\vt_1$, $\vt_2$, we have
\begin{align}
\Prob\sq{| \vt_1\tr\br{\Wt\tr\Xt - \bE[\Wt\tr\Xt]} \vt_2 | \geq t\nrm{\vt_1}\nrm{\vt_2}}
\leq 3 \exp\br{-c\samp \min\cb{\frac{t^2}{\sigma_x^2\sigma_w^2}, \frac{t}{\sigma_x\sigma_w}}}
\label{eq:Chen1}
\end{align}
in particular if $\samp \gtrsim \log \dims$ we have w.h.p.
$$
{|\vt_1\tr \br{\Wt\tr\Xt - \bE[\Wt\tr\Xt]} \vt_2 |} \leq \sigma_x\sigma_w \nrm{\vt_1} \nrm{\vt_2} \sqrt{\frac{\log\dims}{\samp}}
$$
Setting $\vt_1$ to be the first standard basis vector, and using a union bound over $j=1,\ldots,\dims$, we have w.h.p.
$$
\nrm{\br{\Wt\tr\Xt - \bE[\Wt\tr\Xt]} \vt  }_\infty \leq \sigma_x\sigma_w \nrm{\vt} \sqrt{\frac{\log\dims}{\samp}}
$$
holds with probability $1-c_1 \exp(-c_2\log \dims)$ where $c_1,c_2$ are positive constants which are independent of $\sigma_x,\sigma_w, \samp$ and $\dims$.
\end{lem}

\vspace{0.5cm}

\begin{cor}[Modification of Lemma \ref{lem:Chen1} for $\samp=1$] \label{lem:nequals1}
Suppose $\Xt\in\R^{\samp\times k}$ and $\Wt\in\R^{\samp \times m}$ are zero-mean sub-Gaussian matrices with parameters $(\frac{1}{n}\Sigma_x,\frac{1}{n}\sigma^2_x)$, $(\frac{1}{n} \Sigma_w,\frac{1}{n}\sigma^2_w)$ respectively. Then for any fixed vector $\vt_1$ and $\samp=1$ we have w.h.p.
$$
\nrm{\br{\Wt\tr\Xt - \bE[\Wt\tr\Xt]} \vt  }_\infty \leq \sigma_x\sigma_w \nrm{\vt} \log\dims .
$$
\end{cor}
\begin{proof}
Setting $t=c_0\sigma_x\sigma_w\log \dims$, $\samp = 1$ and $\vt$ as the first standard basis vector in Inequality \eqref{eq:Chen1} in Lemma \ref{lem:Chen1} and applying a union bound over $j=1,\ldots,\dims$ yields the result.
\end{proof}

\vspace{0.5cm}

\begin{lem}[Originally Lemma 11 from \cite{Chen:2013te}] \label{lem:inverseCov}
If $\Xt$ and $\Wt$ are zero-mean sub-Gaussian matrices then
\begin{align*}
\Prob\sq{ \sup_{\nrm{\vt_1}=\nrm{\vt_2}=1} | \vt_1\tr\br{\Wt\tr\Xt - \bE[\Wt\tr\Xt]}\vt_2 | \geq t }
\leq
2\exp\br{-c \samp \min(\frac{t^2}{\sigma_x^2\sigma_w^2},\frac{t}{\sigma_x\sigma_w})+6(k+m)}
\end{align*}
In particular, for each $\lambda>0$, if $\samp\gtrsim \max\cb{\frac{\sigma_x^2\sigma_w^2}{\lambda^2},1}(k+m)\log \dims$, then w.h.p.
$$
\sup_{\vt_1,\vt_2} | \vt_1\tr\br{\Wt\tr\Xt - \bE[\Wt\tr\Xt]}\vt_2   | \leq \frac{1}{54}\lambda \nrm{\vt_1}\nrm{\vt_2} .
$$

\end{lem}

\vspace{0.5cm}

\begin{lem}[Quadratic forms of sub-Gaussian random variables. Theorem 2.1 from \cite{Hsu:2011va}] \label{lem:hsu}
Let $\At\in\R^{\samp\times\samp}$ be a matrix, and let $\boldSigma:=\At\tr\At$. $\x$ is a mean-zero random vector such that, for some $\sigma\geq 0$,
$$
\bE \sq{\exp(\alpha \tr \x)}
\leq
\exp(\nrm{\alpha}^2\sigma^2/2)
$$
for all $\alpha\in\R^n$. For all $t>0$
\begin{align*}
\Prob \bigg[ \nrm{\At\x}^2 >  &  \sigma^2\left( \trace{\boldSigma}  + 2\sqrt{\trace{\boldSigma^2}t} \right.  + 2\nrm{\boldSigma}t \bigg)\bigg]
\leq
e^{-t} .
\end{align*}
\end{lem}

\vspace{0.5cm}

\begin{lem}[Extremal singular values of a matrix with i.i.d. sub-Gaussian rows. Theorem 5.39 of \cite{Vershynin:2010vka}] \label{lem:sval}
Let $\At$ be an $\samp\times\dims$ matrix whose rows $\At_i$ are independent sub-Gaussian isotropic random vectors in $\R^\dims$. Then for every $\tau\geq0$, with probability at least $1-2\exp(-c\tau^2)$ we have
$$
\sqrt{\samp} - C\sqrt{\dims} -\tau \leq \sigma_{\samp}(\At) \leq \sigma_{1}(\At) \leq \sqrt{\samp} + C\sqrt{\dims} + \tau
$$
where $C$ and $c$ are constants which depend only on the sub-Gaussian norm of the rows of $\At$.
\end{lem}

\subsection{Discussion} \label{sec:disc}
In this section we provide some additional discussion about the bias and variance of our \iws estimator compared with known results from \cite{Chen:2013te}. We first reproduce the following Lemma
\begin{lem}[Originally Corollary 4 from \cite{Chen:2013te}]
If $\Sigma_w$ is known and $\samp \gtrsim \frac{(1+\sigma_w^2)^2}{\lambda_{\min}(\Sigma_x)\dims\log\dims}$. Then w.h.p., plugging the estimator built using $\widehat{\boldSigma} = \Zt\tr\Zt - \Sigma_w$ and $\hat{\gamma} = \Zt\tr\y$ into Lemma \ref{lem:chenLSbound}, satisfies
\begin{equation}
\nrm{\estbeta - \boldbeta} \lesssim \frac{(\sigma_w^2 + \sigma_w) \nrm{\boldbeta} + \sigma_{\epsilon} \sqrt{1+\sigma_w^2}}{\lambda_{\min}(\Sigma_x)} \sqrt{\frac{\dims \log \dims}{\samp}} ~ . \label{eq:chenIneq1}
\end{equation}
When only an upper bound $\bar{\Sigma}_w \succeq \Sigma_w$ is known then
\begin{align}
\nrm{\estbeta - \boldbeta} \lesssim & \frac{\sq{(\sigma_w^2 + \sigma_w) \nrm{\boldbeta} + \sigma_{\epsilon} \sqrt{1+\sigma_w^2}}   }{\lambda_{\min}(\Sigma_x)- \lambda_{\max}(\bar{\Sigma}_w - \Sigma_w)}\sqrt{\frac{\dims \log \dims}{\samp}}
 + \frac{ \lambda_{\max}(\bar{\Sigma}_w - \Sigma_w)\nrm{\boldbeta}}{\lambda_{\min}(\Sigma_x)- \lambda_{\max}(\bar{\Sigma}_w - \Sigma_w)}
 ~ . \label{eq:chenIneq2}
\end{align}
 \label{lem:chen2}
\end{lem}

\vspace{0.2cm}

We can compare these two statements with our result from Theorem \ref{thm:corruptedls}. Eq. \eqref{eq:chenIneq1} is similar to the bound we have from Theorem \ref{thm:corruptedls} up to the bias term assuming $\pi = 1$ (i.e. all of the points are corrupted). Since we do not use knowledge of $\Sigma_w$ we can compare our result with Eq. \eqref{eq:chenIneq2} which has a bias term which is related to the uncertainty in the estimate of $\Sigma_w$ which in our case is $\sigma_w^2$. It is clear from Lemma \ref{lem:chen2} that the only way to remove this bias completely is to use additional information about the covariance of the corruptions.

\clearpage
\addtocounter{si-sec}{1}
\section{Additional results} \label{sec:addres}
In this section we provide additional empirical results.

\paragraph{Non-corrupted data.}
We first compare performance in three different leverage regimes taken from \cite{Ma:2013tx}: uniform leverage scores (multivariate Gaussian), slightly non-uniform (multivariate-t with 3 degrees of freedom, T-3), highly non-uniform (multivariate-t with 1 degree of freedom, T-1). Full details of the data simulating process can be found in \cite{Ma:2013tx}.

Figures \ref{fig:Mahoney2_est} and \ref{fig:Mahoney2_rmse} show the estimation error and the RMSE respectively for the simulated datasets described in \cite{Ma:2013tx}. The results for the T-3 data are similar to the Gaussian data. The slightly heavier tails of the multivariate $t$ distribution with 3 degrees of freedom cause the leverage scores to be less uniform which degrades the performance of uniform subsampling relative to \srht and \our. Figure \ref{fig:Mahoney2_rmse} shows that the RMSE performance is similar to that of the statistical estimation error.


\begin{figure}[htp]
\begin{centering}
\subfloat[{Gaussian}]{
\begin{minipage}{0.5\textwidth}
    \includegraphics[width=0.98\textwidth, keepaspectratio=true]{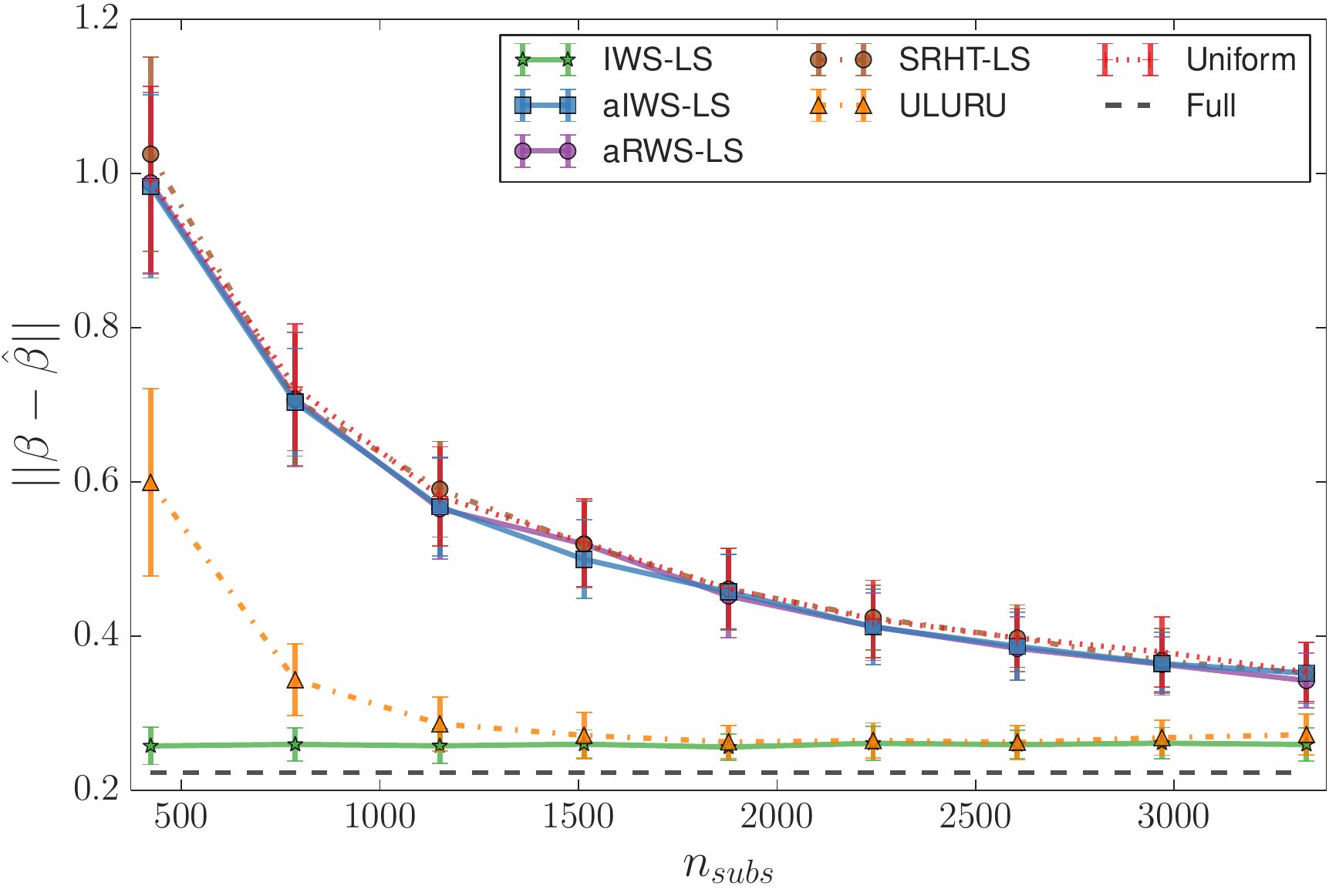}

\end{minipage}
}
\subfloat[{T-3}]{
\begin{minipage}{0.5\textwidth}
    \includegraphics[width=0.98\textwidth, keepaspectratio=true]{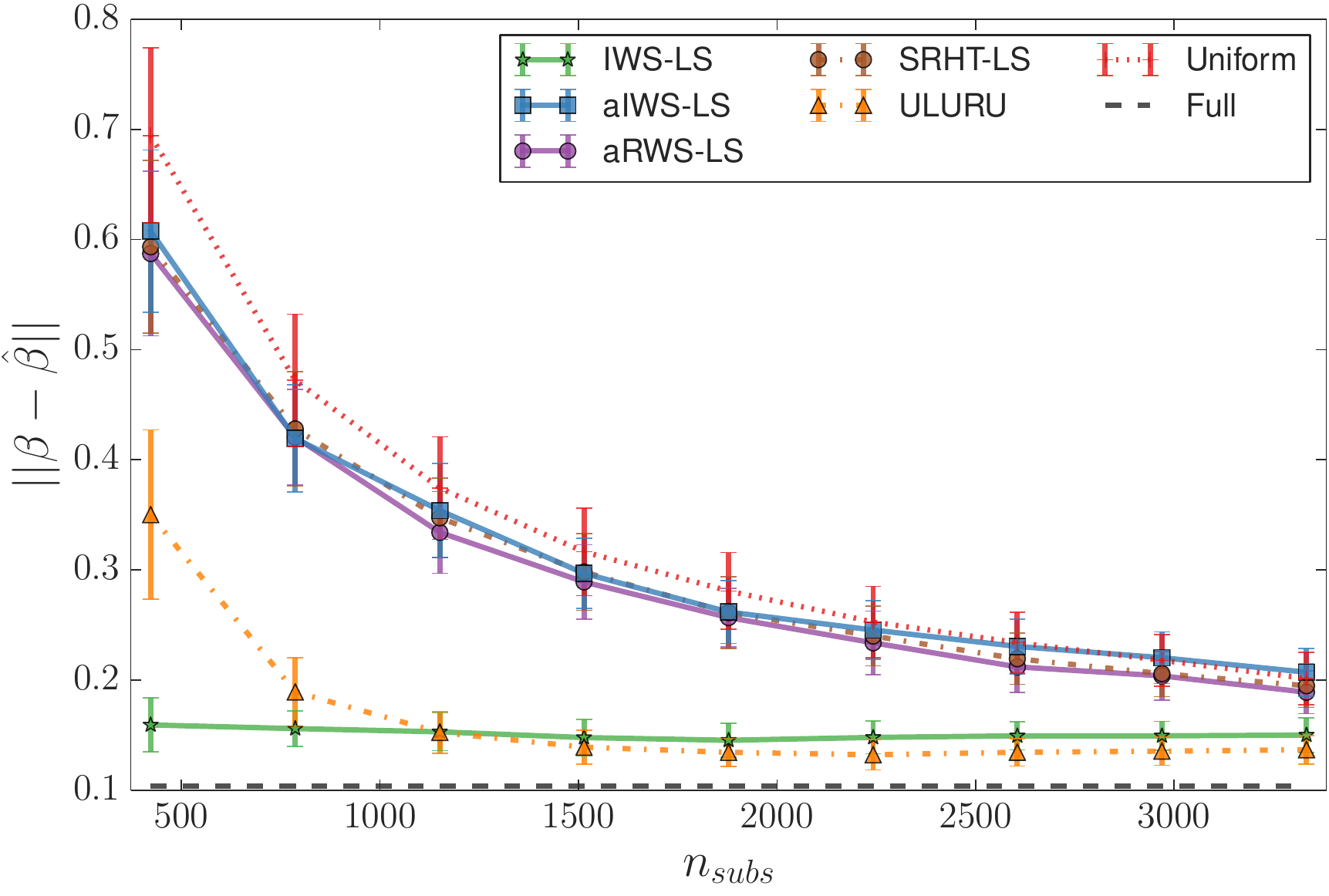}

\end{minipage}
}

\subfloat[{T-1}]{
\begin{minipage}{0.5\textwidth}
    \includegraphics[width=0.98\textwidth, keepaspectratio=true]{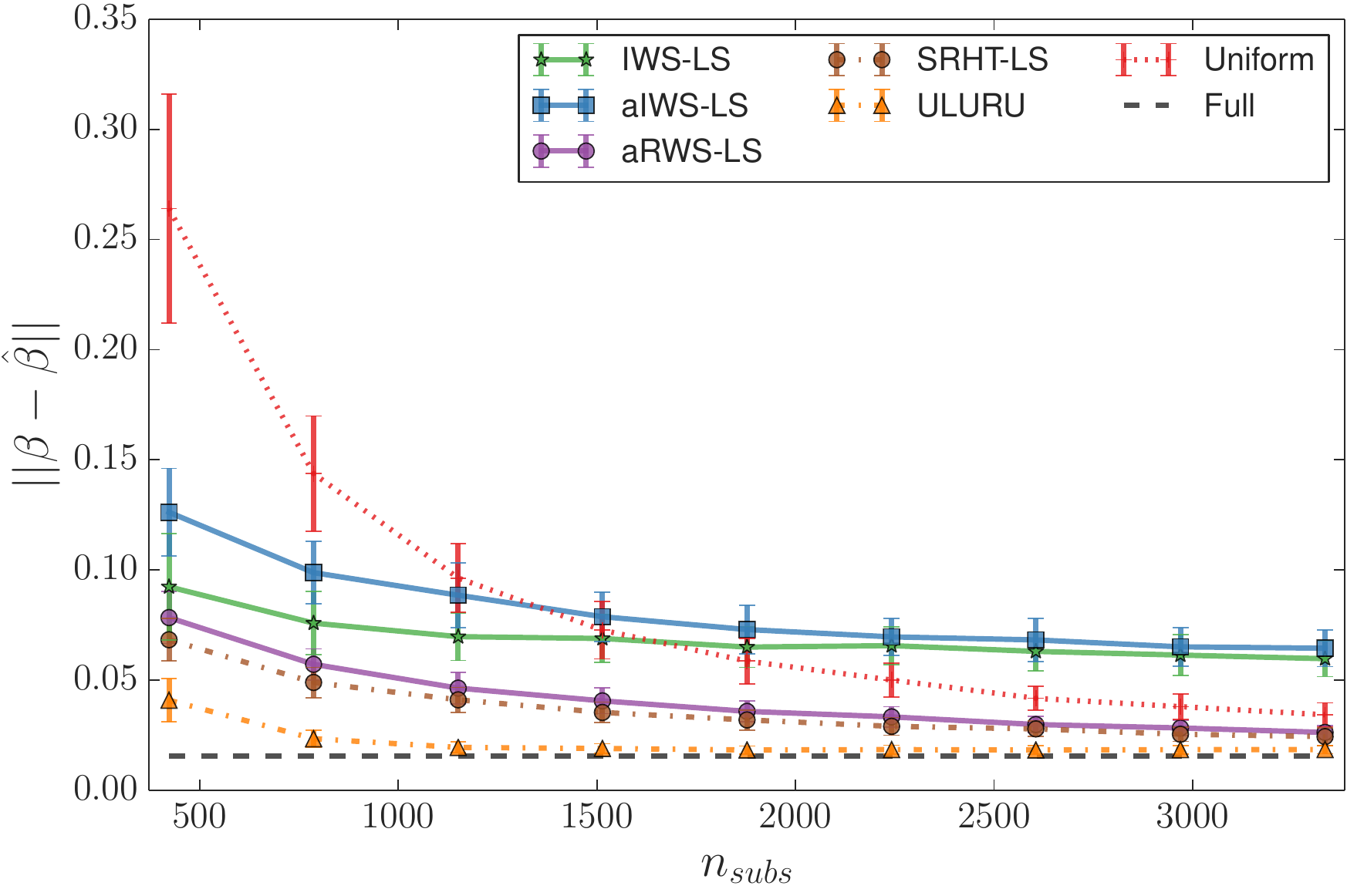}

\end{minipage}
}
\caption{Comparison of mean estimation error and standard deviation on a selection of non-corrupted datasets. \label{fig:Mahoney2_est}}
\end{centering}
\end{figure}

\begin{figure}[htp]
\begin{centering}
\subfloat[{Gaussian}]{
\begin{minipage}{0.5\textwidth}

     \includegraphics[width=0.98\textwidth, keepaspectratio=true]{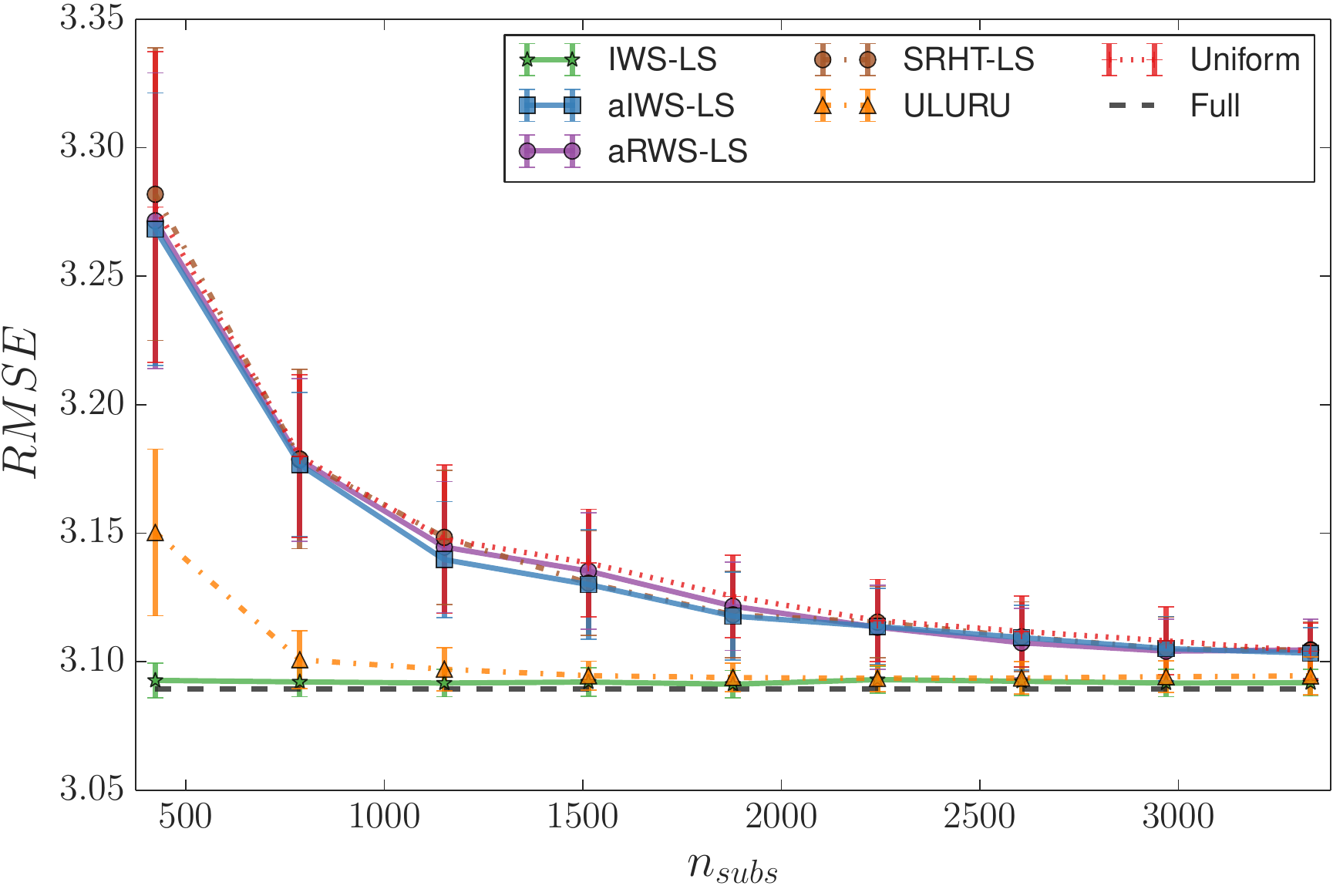}
\end{minipage}
}
\subfloat[{T-3}]{
\begin{minipage}{0.5\textwidth}

     \includegraphics[width=0.98\textwidth, keepaspectratio=true]{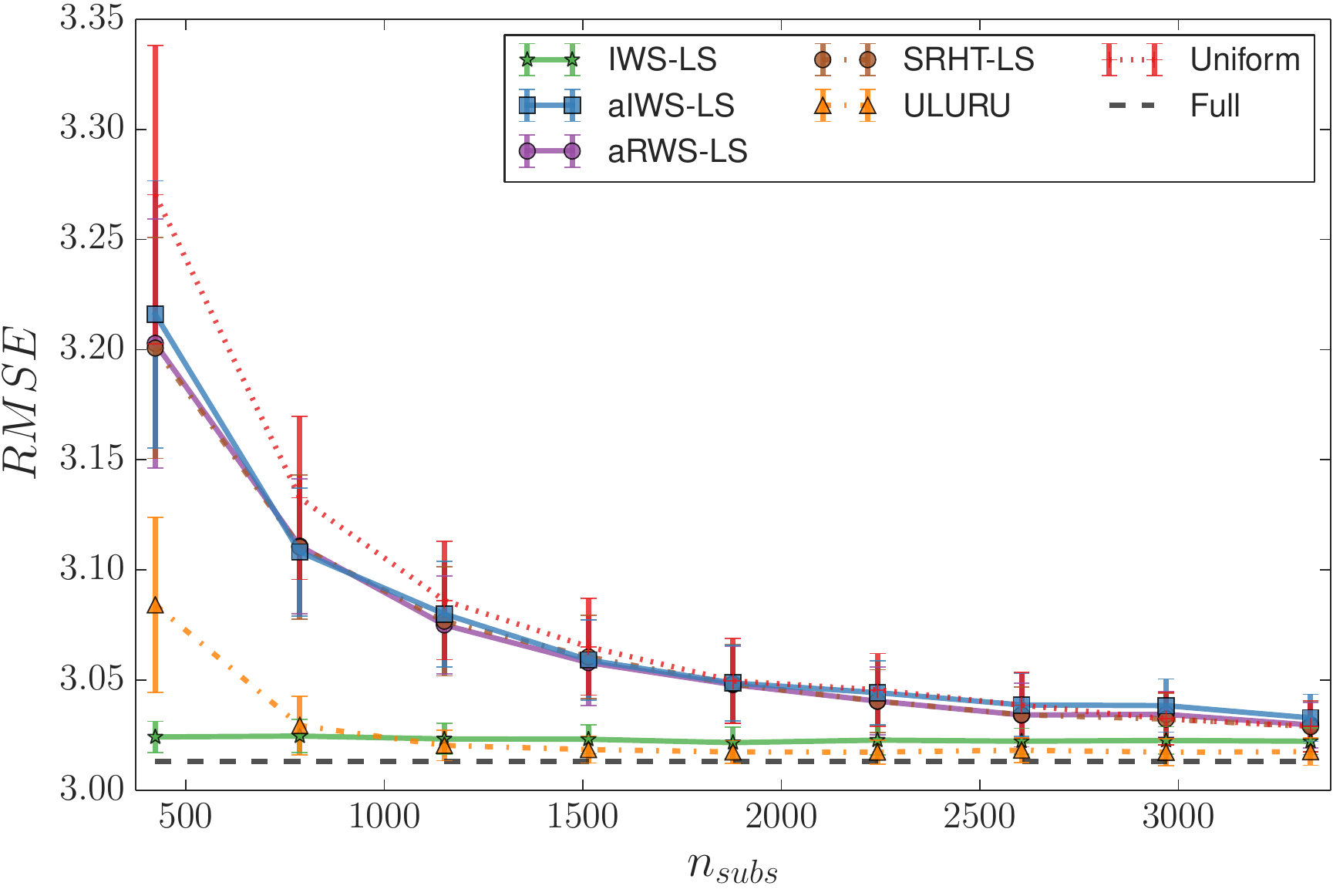}
\end{minipage}
}

\subfloat[{T-1}]{
\begin{minipage}{0.5\textwidth}

     \includegraphics[width=0.98\textwidth, keepaspectratio=true]{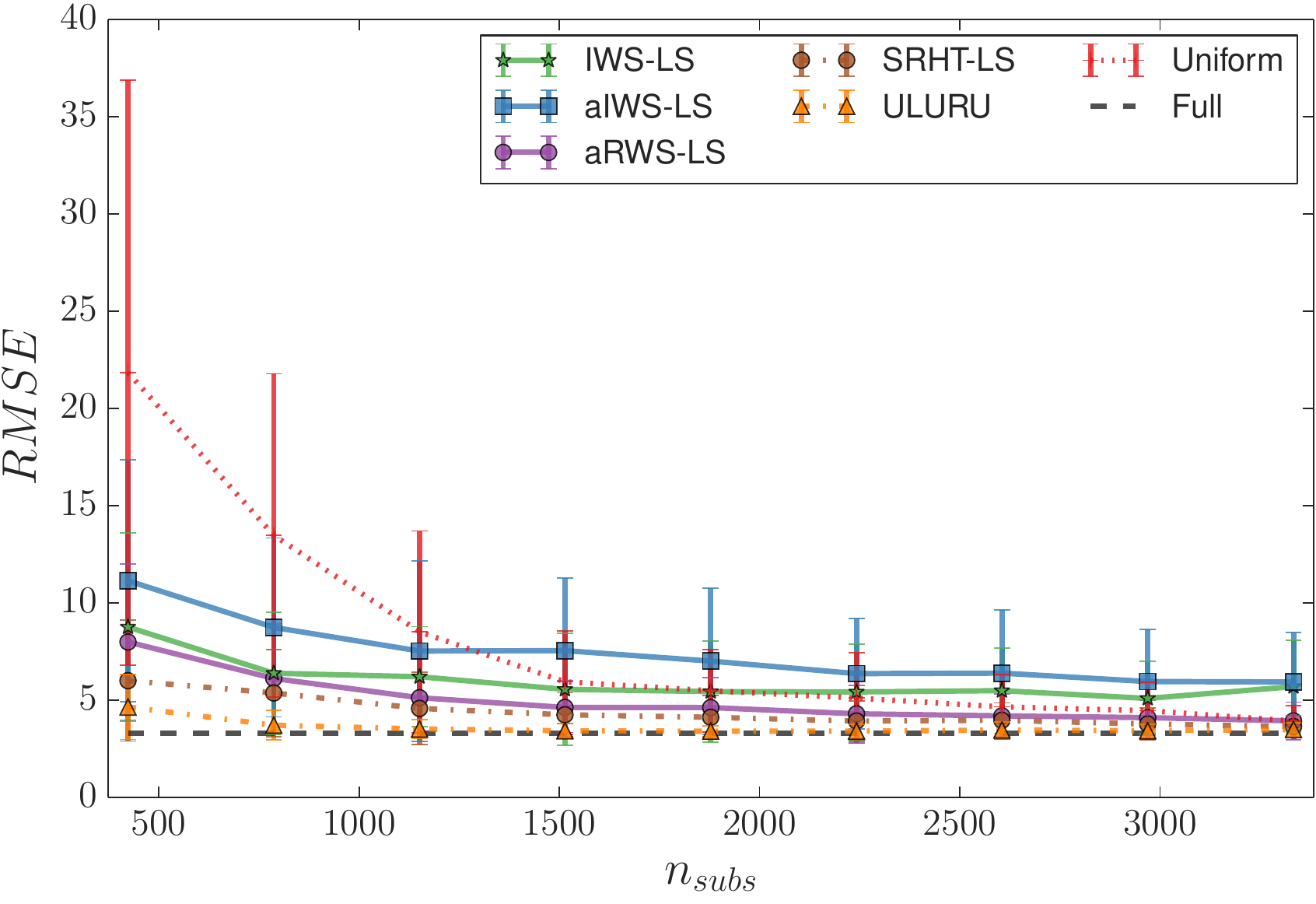}
\end{minipage}
}
\caption{Comparison of root mean squared  prediction error (RMSE) and standard deviation on a selection of non-corrupted datasets. \label{fig:Mahoney2_rmse}}
\end{centering}
\end{figure}

\paragraph{Corrupted data.}
Figures \ref{fig:Corrupted2_est} and \ref{fig:Corrupted2_rmse} show the estimation error and RMSE respectively for the corrupted simulated datasets. In all settings influence based methods outperform all other approximation methods. For $5\%$ corruptions for a small number of samples \uluru outperforms the other subsampling methods. However, as the  number of samples increases, influence based methods start to outperform OLS. For $>3000$ subsamples, the bias correction step of \uluru causes it to diverge from OLS and ultimately perform worse than uniform.

 For $10\%$ corruptions, \ourapprox and \irh converge quickly to \our. As the number of corruptions increase further, the relative performance of \our with respect to OLS decreases slightly as suggested by Remark \ref{rem:limit}.

For $30\%$ corruptions, the approximate influence algorithms achieve almost exactly the same performance as \our. Even for a small number of samples all of the influence methods far outperform \OLS. As the proportion of corruptions increases further, the rate at which the approximate influence algorithms approach \our slows and the relative difference between \our and OLS decreases slightly. In all cases, influence based methods achieve lower-variance estimates.  Here, \uluru converges quickly to the OLS solution but is not able to overcome the bias introduced by the corrupted datapoints.

\begin{figure*}[htp]
\begin{centering}
\subfloat[{5\% Corruptions}]{
\begin{minipage}{0.5\textwidth}
    \includegraphics[width=0.95\textwidth, keepaspectratio=true]{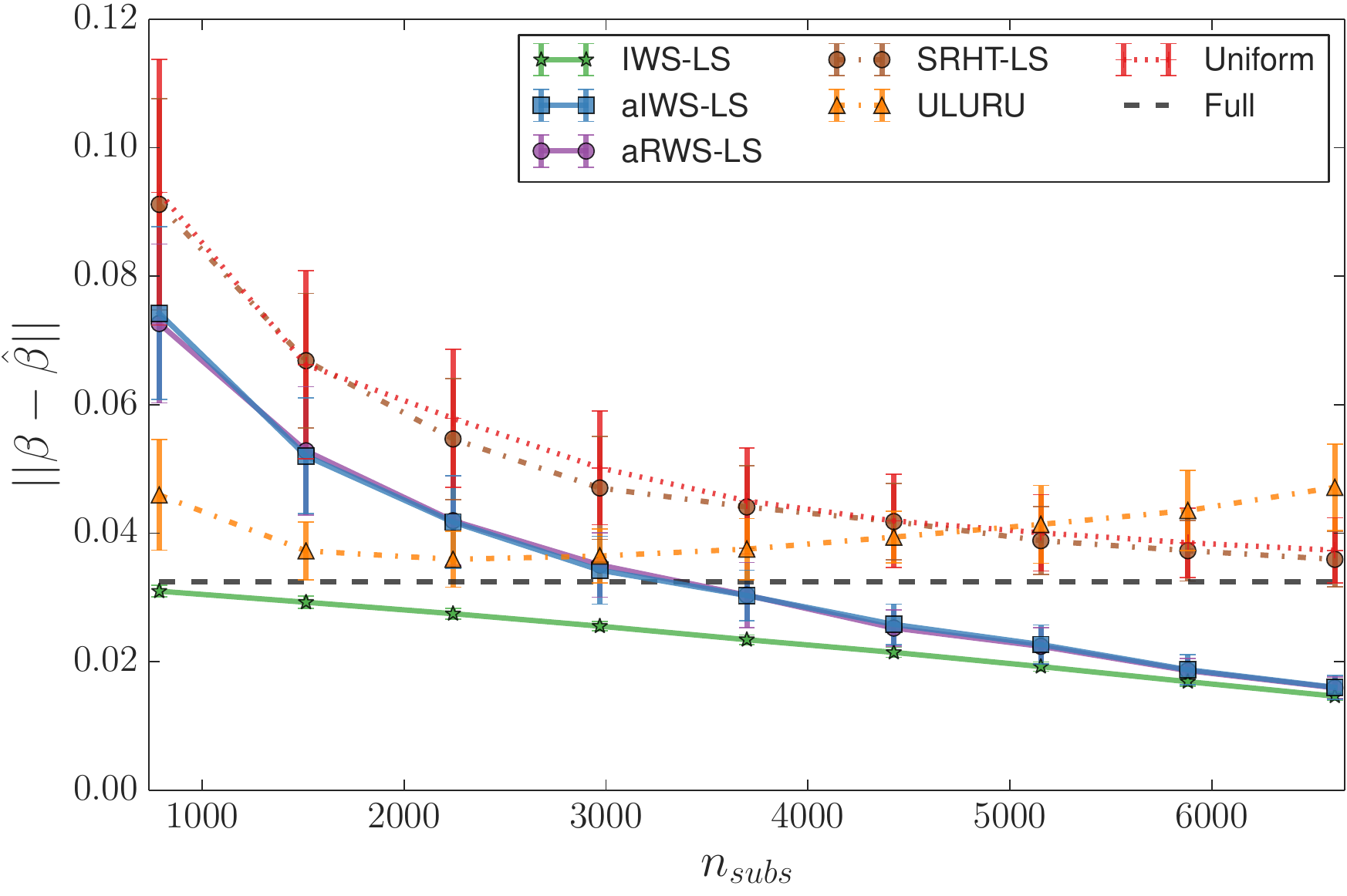}

\end{minipage}
}
\subfloat[{10\% Corruptions}]{
\begin{minipage}{0.5\textwidth}
    \includegraphics[width=0.95\textwidth, keepaspectratio=true]{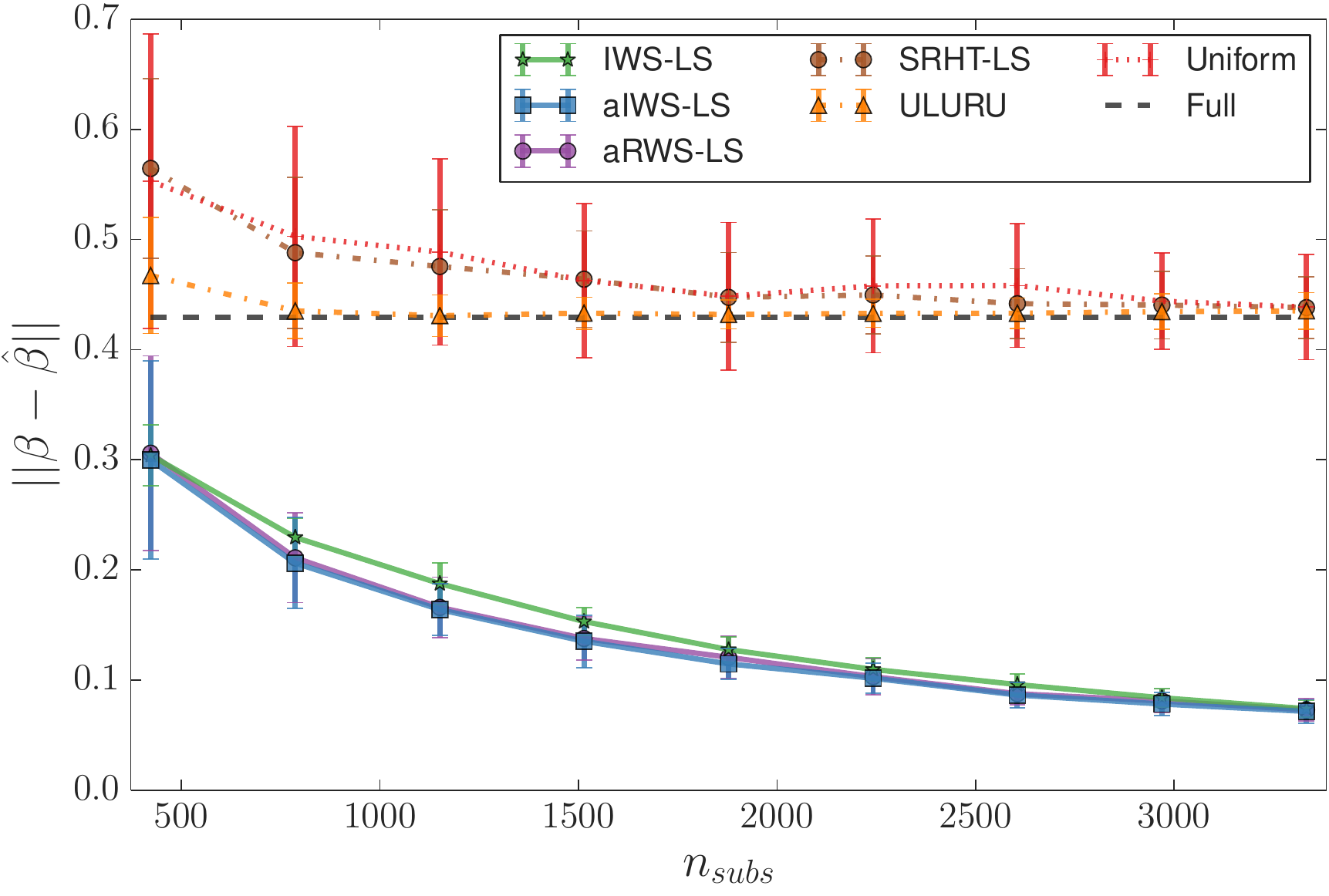}

\end{minipage}
}

\subfloat[{30\% Corruptions}]{
\begin{minipage}{0.5\textwidth}
    \includegraphics[width=0.95\textwidth, keepaspectratio=true]{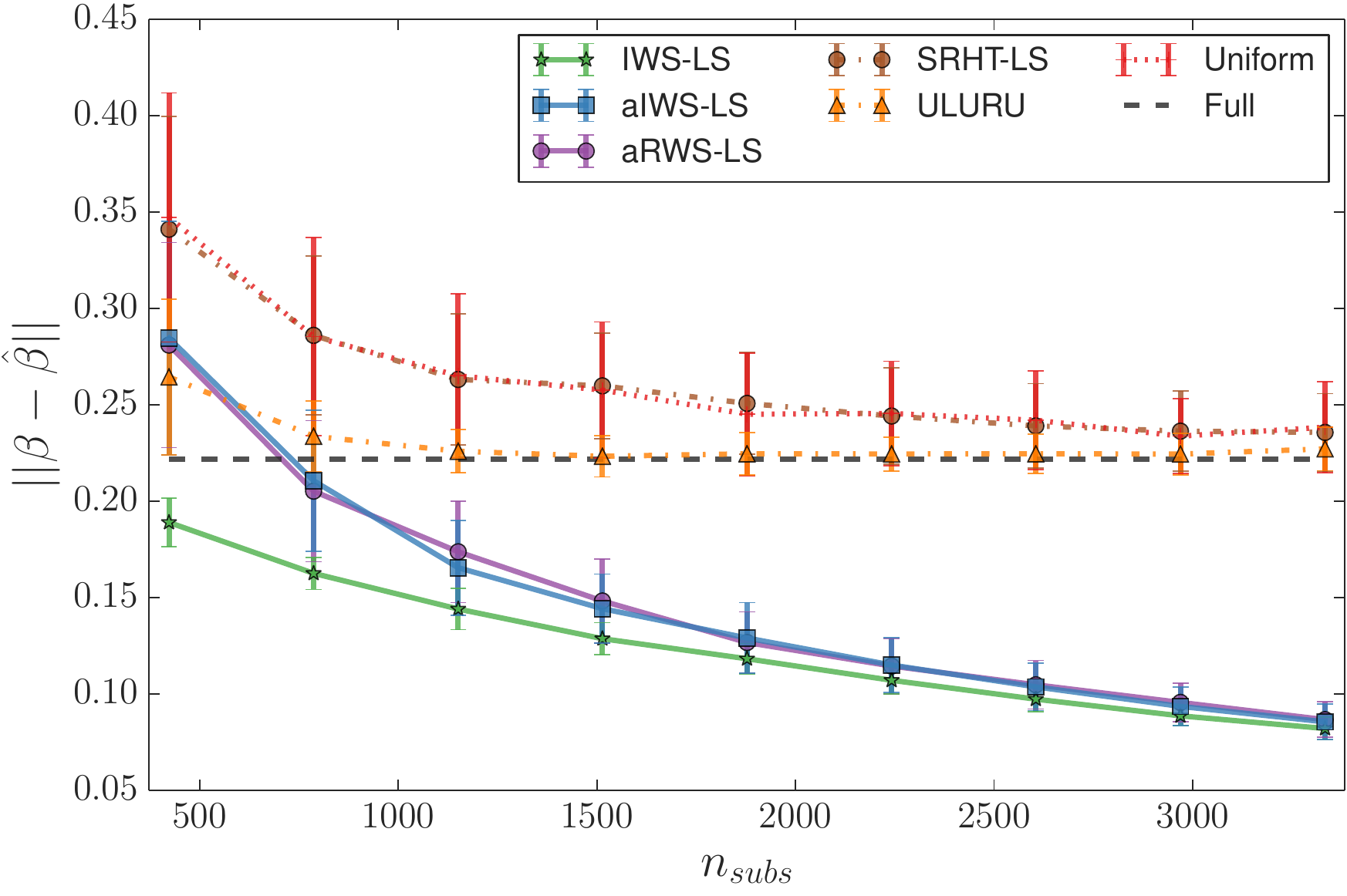}

\end{minipage}
}
\caption{Comparison of mean estimation error and standard deviation on a selection of corrupted datasets. \label{fig:Corrupted2_est}}
\end{centering}
\end{figure*}

\begin{figure*}[htp]
\begin{centering}
\subfloat[{5\% Corruptions}]{
\begin{minipage}{0.5\textwidth}

     \includegraphics[width=0.95\textwidth, keepaspectratio=true]{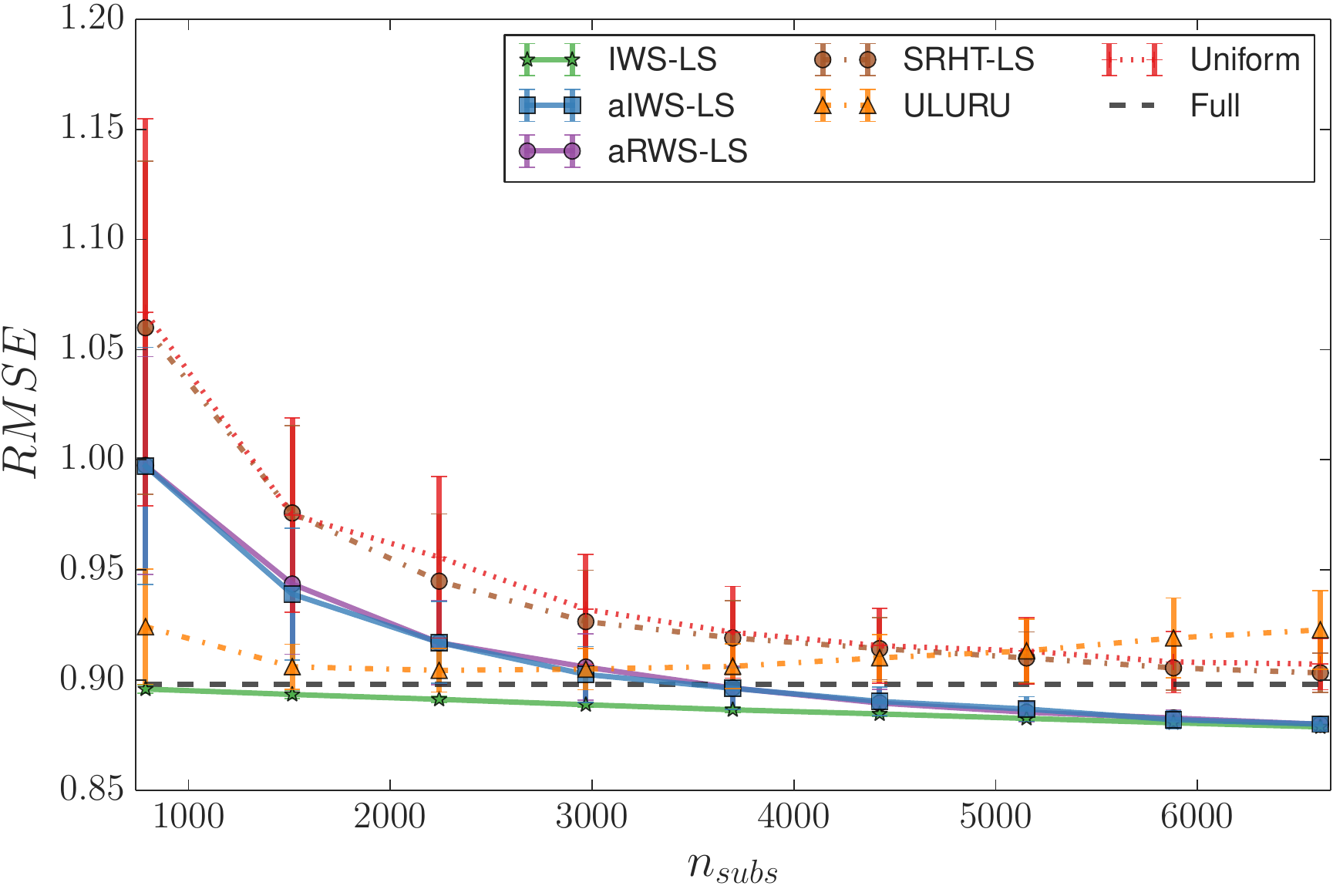}
\end{minipage}
}
\subfloat[{10\% Corruptions}]{
\begin{minipage}{0.5\textwidth}

     \includegraphics[width=0.95\textwidth, keepaspectratio=true]{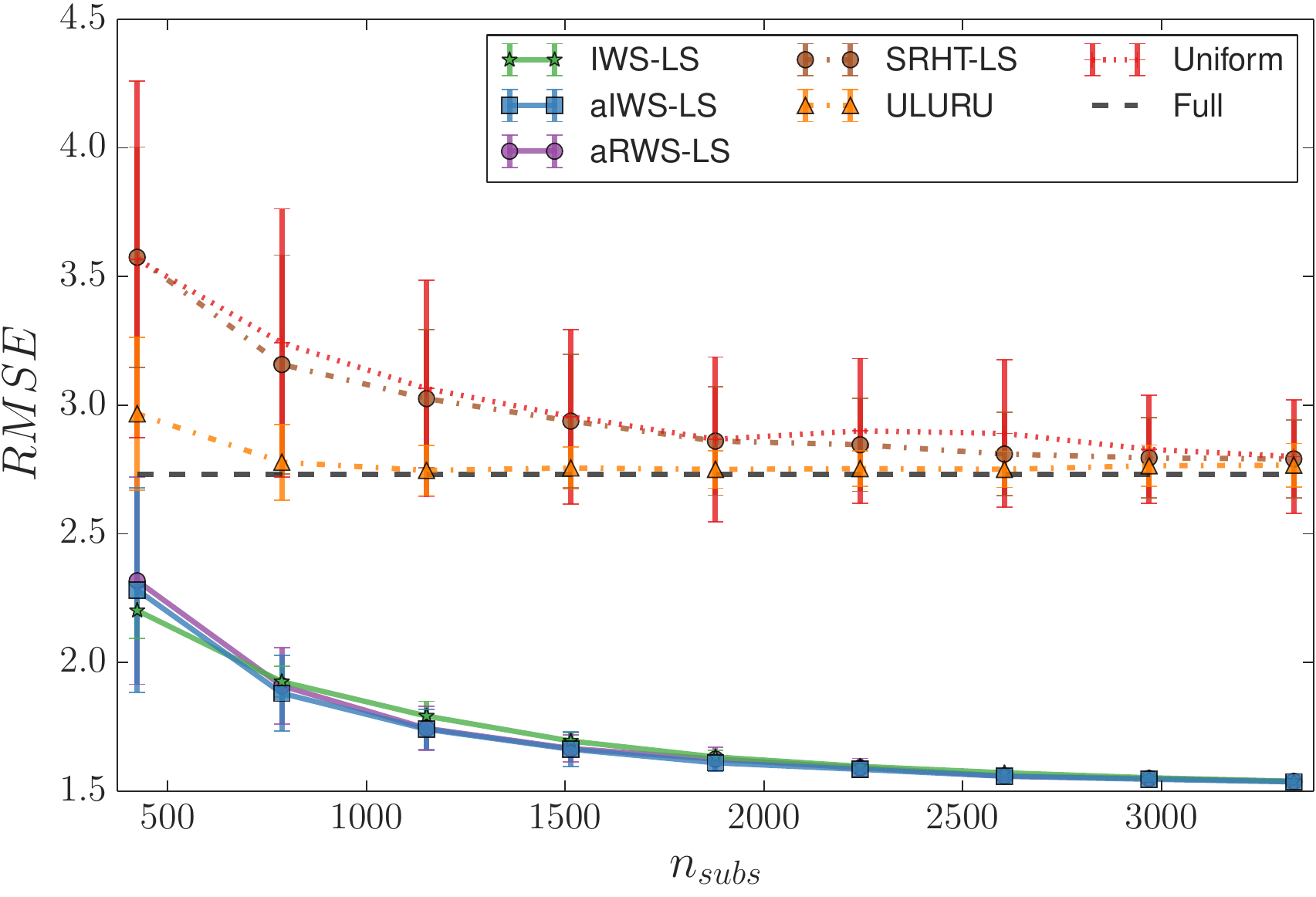}
\end{minipage}
}

\subfloat[{30\% Corruptions}]{
\begin{minipage}{0.5\textwidth}

\includegraphics[width=0.95\textwidth, keepaspectratio=true]{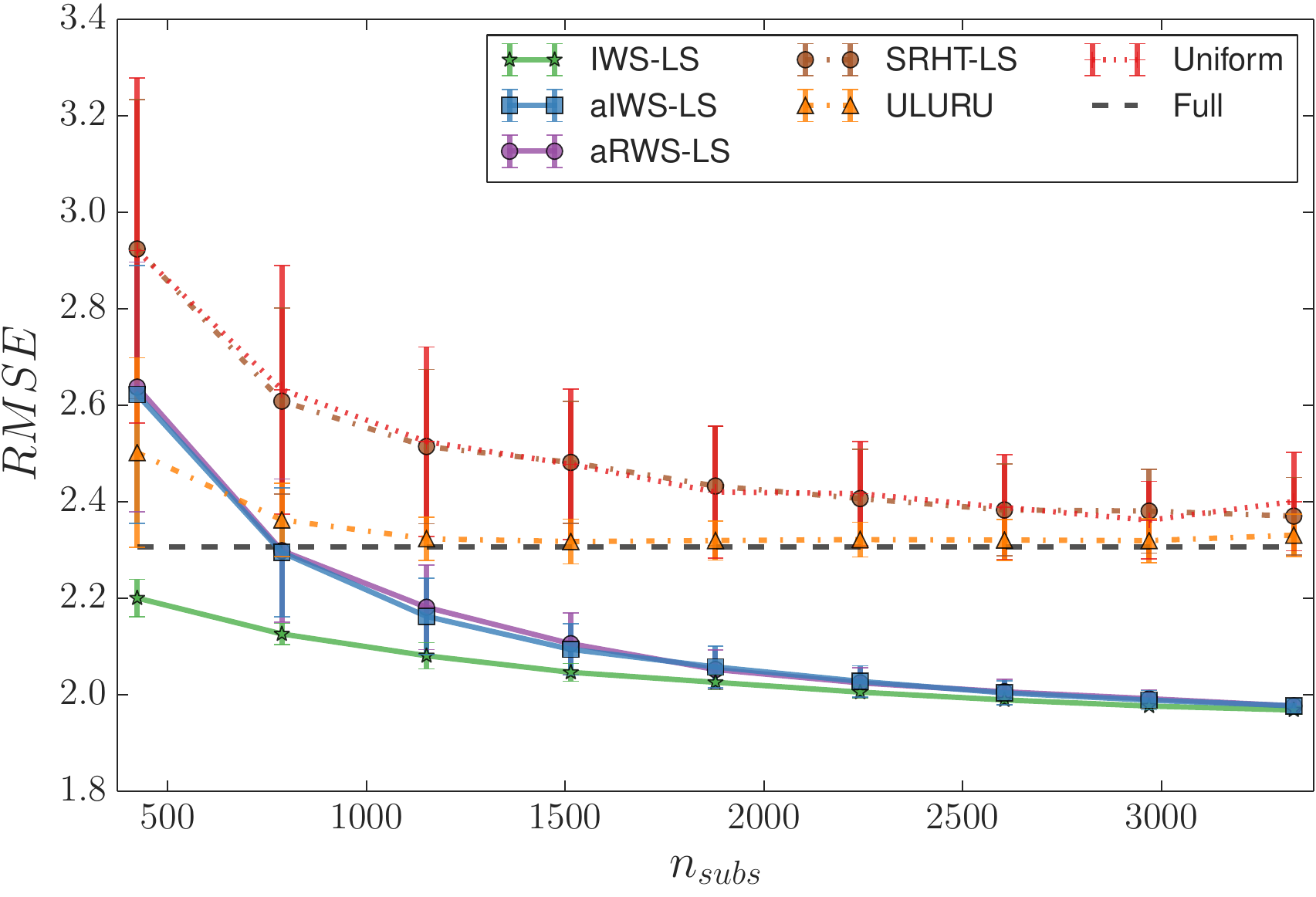}
\end{minipage}
}
\caption{Comparison of test RMSE and standard deviation on a selection of corrupted datasets. \label{fig:Corrupted2_rmse}}
\end{centering}
\end{figure*}

\paragraph{Larger Scale Experiments Corrupted data.}
We now present results on larger scale simulated data. We used the same experimental setup as in \S\ref{sec:results} but we increase the size of the data to $\samp=100,000$ and $\dims=500$. 

Figures \ref{fig:Corrupted_large} and \ref{fig:Corrupted_large_est} show the estimation error and RMSE respectively. In this setting, computing \our is too slow (due to the exact leverage computation) so we omit the results but we notice that \ourapprox and \irh quickly improve over the full least squares solution and the other randomized approximations. The general trend in this setting is the same as with the smaller experiments, however for $5\%$ corruptions the improvement of \ourapprox and \irh over OLS happens with a much smaller subsampling ratio than with smaller datasets. 

\begin{figure*}[htp]
\begin{centering}
\subfloat[{5\% Corruptions}]{
\begin{minipage}{0.5\textwidth}
    \includegraphics[width=0.95\textwidth, keepaspectratio=true]{nips_figure_corrupted-n100000-p500-corruption-n_corrupt5000_norm_diff}
\end{minipage}
}
\subfloat[{10\% Corruptions}]{
\begin{minipage}{0.5\textwidth}
    \includegraphics[width=0.95\textwidth, keepaspectratio=true]{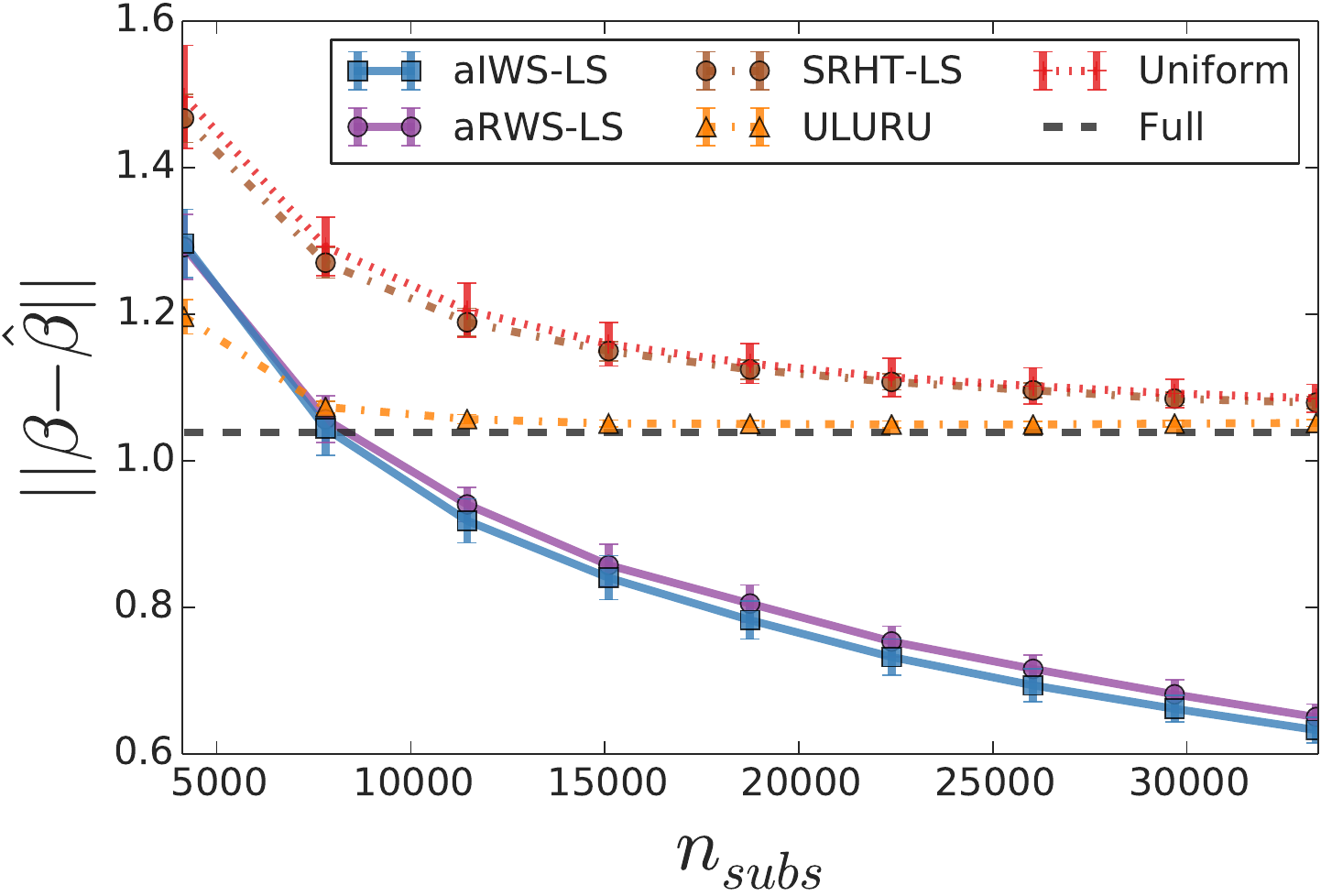}
\end{minipage}
}

\subfloat[{30\% Corruptions}]{
\begin{minipage}{0.5\textwidth}
    \includegraphics[width=0.95\textwidth, keepaspectratio=true]{nips_figure_corrupted-n100000-p500-corruption-n_corrupt30000_norm_diff}
\end{minipage}
}
\caption{Comparison of mean estimation error and standard deviation on a selection of corrupted datasets. \label{fig:Corrupted_large}}
\end{centering}
\end{figure*}

\begin{figure*}[htp]
\begin{centering}
\subfloat[{5\% Corruptions}]{
\begin{minipage}{0.5\textwidth}
    \includegraphics[width=0.95\textwidth, keepaspectratio=true]{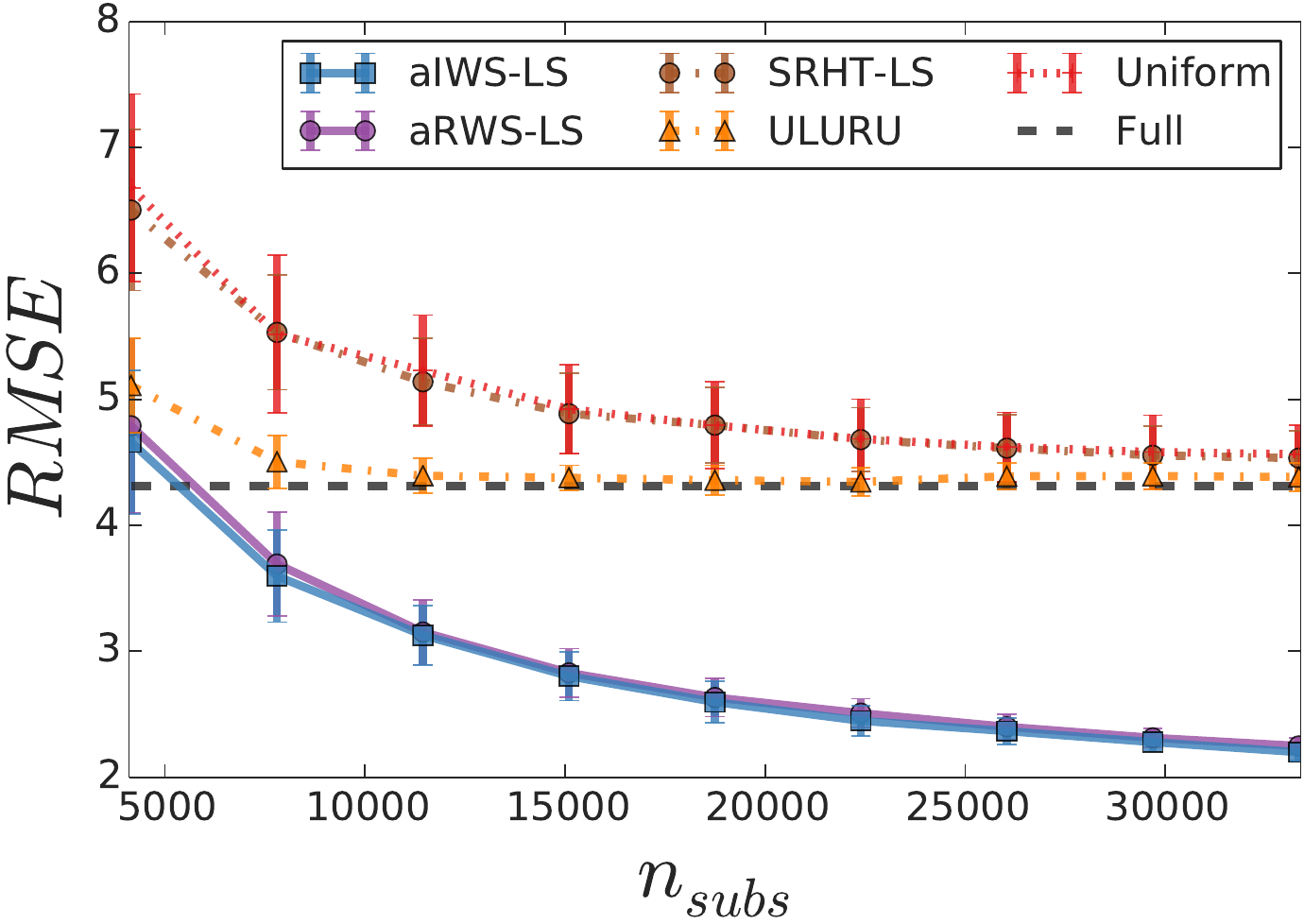}
\end{minipage}
}
\subfloat[{10\% Corruptions}]{
\begin{minipage}{0.5\textwidth}
    \includegraphics[width=0.95\textwidth, keepaspectratio=true]{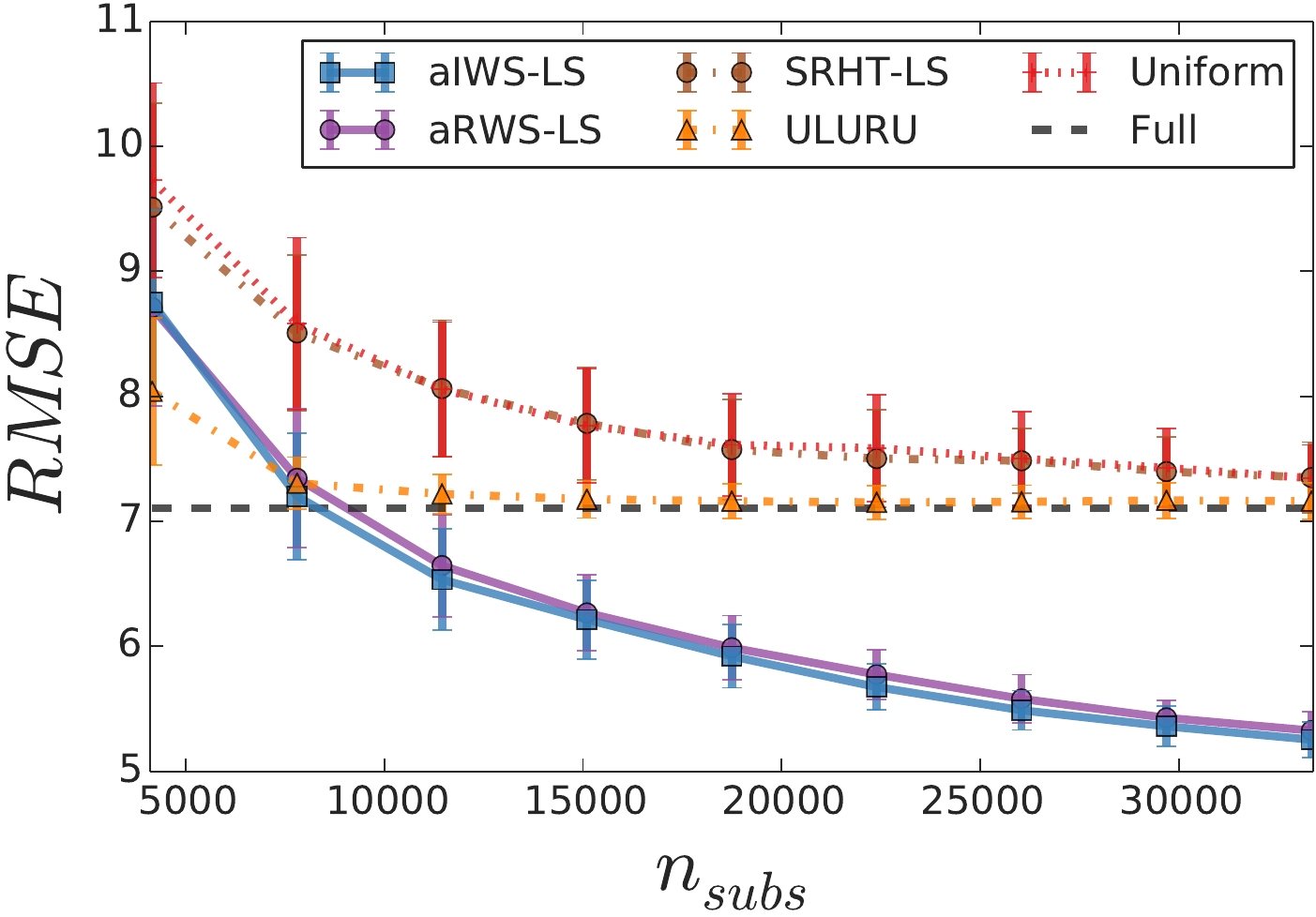}
\end{minipage}
}

\subfloat[{30\% Corruptions}]{
\begin{minipage}{0.5\textwidth}
    \includegraphics[width=0.95\textwidth, keepaspectratio=true]{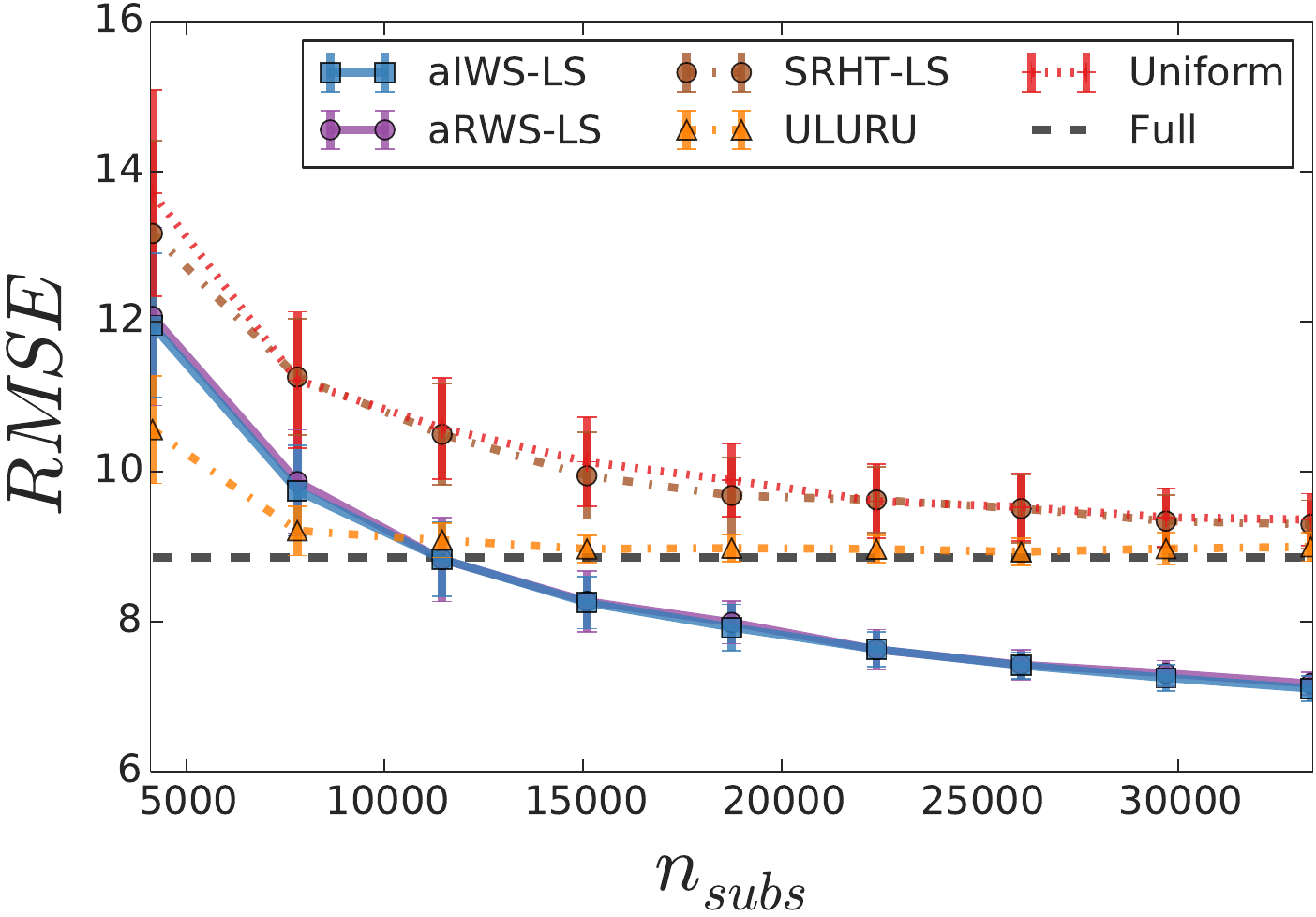}
\end{minipage}
}
\caption{Comparison of test RMSE and standard deviation on a selection of corrupted datasets. \label{fig:Corrupted_large_est}}
\end{centering}
\end{figure*}

\end{document}